\documentclass{article}

\usepackage[preprint]{neurips_2026}

\usepackage[utf8]{inputenc} 
\usepackage[T1]{fontenc}    
\usepackage{hyperref}       
\usepackage{url}            
\usepackage{booktabs}       
\usepackage{amsfonts}       
\usepackage{nicefrac}       
\usepackage{microtype}      
\usepackage{xcolor}         
\usepackage{booktabs}
\usepackage{xr}
\usepackage{amsmath, amssymb, amsthm}
\usepackage{mathtools}
\newtheorem{theorem}{Theorem}
\newtheorem{lemma}{Lemma}
\newtheorem{proposition}{Proposition}
\usepackage{wrapfig}
\usepackage{float}
\usepackage{enumitem}

\title{KC-3DGS: Kurtosis-Constrained Gaussian Splatting for High-Fidelity View Synthesis}

\author{%
  Vivekjyoti Banerjee\\
  Johns Hopkins University\\
  \texttt{vivekjyoti24@gmail.com} \And
  Abhay Yadav \\
  Johns Hopkins University\\
  \texttt{jai.abhayk@gmail.com} \And
    Rama Chellappa\\
  Johns Hopkins University\\
  \texttt{rchella4@jhu.edu} \And
    Aniket Roy \\
  NEC Labs America\\
  \texttt{ank.roy4@gmail.com}
}

\begin{document}

\maketitle

\begin{abstract}
3D Gaussian Splatting (3DGS) enables real-time novel view synthesis by representing scenes as collections of anisotropic Gaussians optimized via differentiable rasterization. However, standard pixel-space losses (L1, SSIM) constrain only aggregate reconstruction error, permitting the optimization to redistribute error across frequency scales. This leads to oversmoothing and structural artifacts, particularly in sparse-view settings where supervision is limited. We propose KC-3DGS, which augments 3DGS training with wavelet-domain supervision based on natural image statistics. Our method combines three components: (1) a multi-scale wavelet coefficient alignment loss that explicitly penalizes missing high-frequency detail, (2) a supervised kurtosis concentration loss that encourages rendered images to match the heavy-tailed frequency statistics of ground-truth images, and (3) a cross-band covariance penalty that promotes frequency specialization. We provide theoretical analysis showing that pixel-space losses admit a family of indistinguishable perturbations under wavelet redistribution, and that our joint objective excludes degenerate solutions. Experiments across MipNeRF360, Tanks\&Temples, MVImgNet, DeepBlending, and WRIVA-ULTRRA demonstrate consistent improvements in perceptual quality. On the challenging WRIVA-ULTRRA outdoor dataset, KC-3DGS achieves a 9.48\% improvement in DreamSim while also improving PSNR, SSIM, and LPIPS. In sparse-view settings with only 12 training images, our method improves PSNR by up to 0.5 dB on MipNeRF360 while maintaining perceptual quality. The approach integrates seamlessly into existing 3DGS pipelines as a plug-and-play regularization strategy.
\end{abstract}


\section{Introduction}
\label{sec:introduction}
Novel-view synthesis (NVS) aims to render photorealistic images of a scene from unseen viewpoints, and underpins applications in VR/AR, robotics, and content creation \cite{mildenhall2020nerf,mipnerf360}. Recent explicit representations, in particular 3D Gaussian Splatting (3DGS), model scenes as sets of anisotropic Gaussians optimized to approximate a radiance field, enabling real-time, high-quality rendering via differentiable rasterization \cite{kerbl2023gaussians}. Compared to NeRF-based methods, 3DGS achieves competitive or superior image quality with much faster training and inference and has rapidly become a strong baseline for real-time NVS \cite{mildenhall2020nerf,kerbl2023gaussians}.

However, in practical settings only a small number of views are often available, making the reconstruction problem severely underconstrained \cite{jain2021dietnerf,niemeyer2022regnerf,wang2023sparsenerf}. Under such sparse-view supervision, 3DGS suffers from unstable optimization and artifacts \cite{xiong2023sparsegs,dngaussian,fsgs,corgs}. Adaptive density control (ADC) can misinterpret incomplete gradients, spawning Gaussians in regions unsupported by true geometry and pruning others prematurely, leading to floaters, oversmoothed blobs, and loss of fine details \cite{kerbl2023gaussians,dngaussian,fsgs}. Standard photometric objectives (e.g., L1, L2, SSIM \cite{wang2004image}) penalize per-pixel errors but do not effectively constrain higher-order structure, so models can achieve strong PSNR/SSIM while still producing overly smooth, structurally inconsistent renderings \cite{zhang2018unreasonable}.

Natural image statistics offer a complementary perspective. When images are decomposed into multi-scale, oriented band-pass components (e.g., via wavelets), the resulting coefficients exhibit sparse, heavy-tailed distributions with characteristic kurtosis patterns that are remarkably consistent across subbands \cite{field1987relations,ruderman1994statistics,simoncelli2001natural,mallat1989theory,daubechies1992ten}. These kurtosis “concentration” properties have been widely exploited in classical image processing and, more recently, as priors for improving the perceptual quality of generative models \cite{olshausen1996emergence,simoncelli1996noise,chang2000adaptive,portilla2000parametric,diffnat}. Motivated by this, we ask: can we regularize explicit 3D radiance representations by aligning their wavelet-domain statistics with those of real images?

We propose KC-3DGS, a kurtosis-constrained Gaussian splatting framework for high-fidelity view synthesis. KC-3DGS augments the standard 3DGS objective with a lightweight, fully differentiable wavelet-based kurtosis concentration (KC) loss \cite{kerbl2023gaussians,diffnat}. We apply a multi-scale Daubechies-3 wavelet transform to both rendered and ground-truth images, compute kurtosis per subband, and penalize discrepancies in both the global kurtosis range and the per-band values \cite{mallat1989theory,daubechies1992ten}. This regularizer reduces pathological cross-band kurtosis imbalances while preserving the inherently heavy-tailed nature of natural images, encouraging sharper textures, cleaner edges, and more coherent structures—without modifying the 3DGS architecture or relying on additional geometric supervision \cite{field1987relations,simoncelli2001natural}.

KC-3DGS is complementary to existing geometric, depth-based, or generative priors for robust NVS, and integrates seamlessly into existing 3DGS training pipelines \cite{dngaussian,fsgs,corgs,wu2024reconfusion}. We evaluate our approach on MipNeRF360 \cite{mipnerf360}, Tanks and Temples \cite{Knapitsch2017}, MVImgNet \cite{yu2023mvimgnet}, and a sparse-view iPad LTS dataset, focusing on challenging regimes with as few as 12 training views per scene. Across these benchmarks, KC-3DGS consistently improves standard metrics (PSNR, SSIM \cite{wang2004image}, LPIPS \cite{zhang2018unreasonable}, DreamSim \cite{fu2023dreamsim}), reduces floaters and oversmoothing, and yields more stable reconstructions. Analyzing the evolution of wavelet-band kurtosis during training further shows that KC regularization stabilizes 3DGS optimization and promotes more compact, faithful scene representations.
\begin{figure}[!t]
    \centering
    \includegraphics[width=0.92\textwidth]{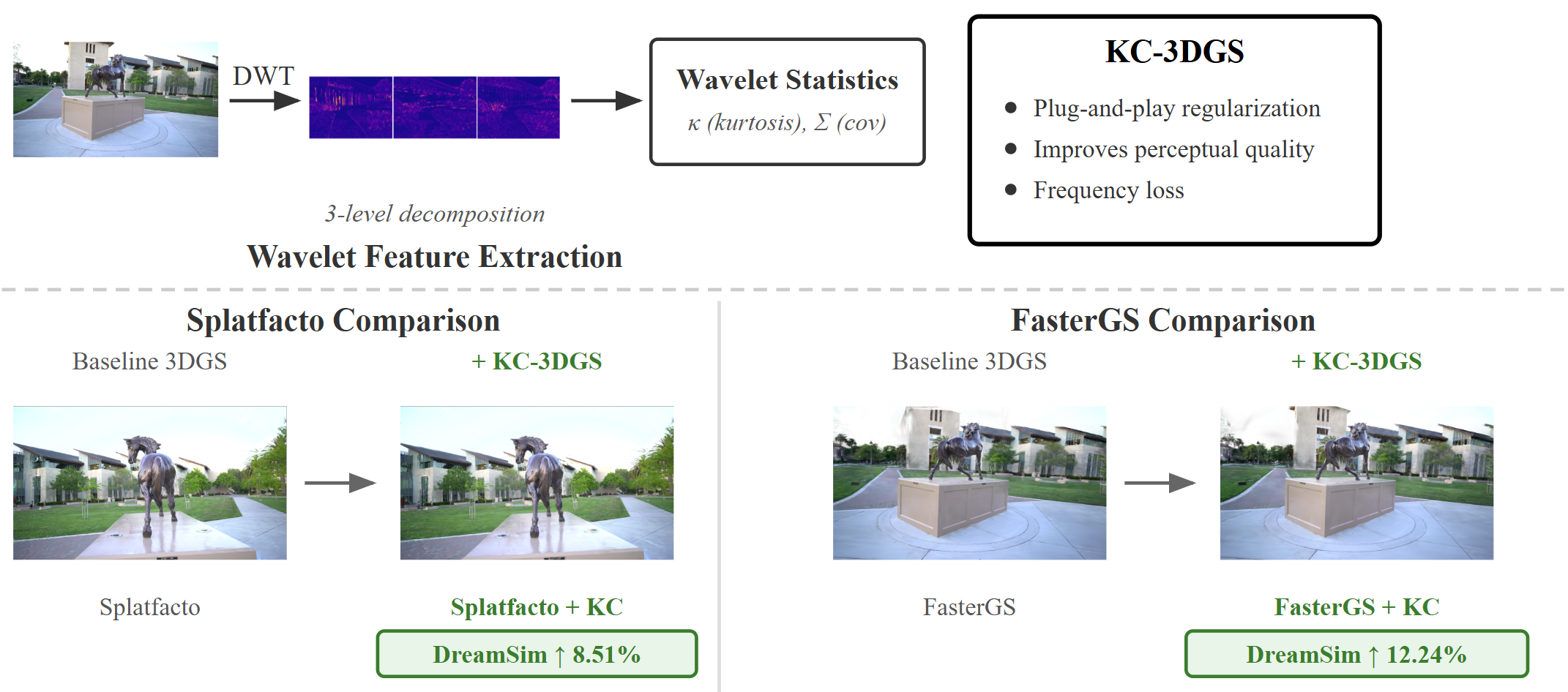}
    \caption{\textbf{Overview of KC-3DGS.} \textit{Top}: Our method extracts wavelet features via a 3-level discrete wavelet transform (DWT) \cite{mallat1989theory,daubechies1992ten}, computes wavelet statistics (kurtosis $\kappa$ and cross-band covariance $\Sigma$), and applies frequency-domain losses as a plug-and-play regularization for 3D Gaussian Splatting \cite{kerbl2023gaussians}. Bottom: Visual comparisons demonstrating consistent perceptual improvements across different 3DGS implementations. Adding KC-3DGS regularization to Splatfacto \cite{tancik2023nerfstudio} improves DreamSim by 8.51\%, while integration with FasterGS \cite{hahlbohm2026fastergs} yields a 12.24\% improvement from some viewpoints, highlighting the method's effectiveness as a general-purpose enhancement for existing 3DGS pipelines.}
    \label{fig:teaser}
    \vspace{-0.4cm}
\end{figure}
Our contributions are threefold:
\begin{itemize}
    \item We introduce KC-3DGS, a plug-and-play wavelet-domain kurtosis concentration loss for 3DGS that explicitly aligns higher-order statistics of rendered images with natural scenes.
    \item We propose a perceptually guided regularization strategy that operates on multi-scale wavelet coefficients rather than only pixel-space errors, improving structural coherence and reducing artifacts such as floaters and oversmoothed regions in sparse-view settings.
    \item We provide a comprehensive empirical and statistical analysis on several challenging NVS benchmarks, demonstrating that enforcing natural image statistics in the wavelet domain is an effective and general mechanism for enhancing the perceptual quality and robustness of explicit radiance representations.
\end{itemize}
\section{Related Work} 
\label{sec:realted-work}
\vspace{-0.2cm}


\textbf{Sparse-view novel-view synthesis and 3DGS regularization.} Sparse-view NVS is severely underconstrained because a small number of posed images
does not fully determine scene geometry, visibility, or high-frequency appearance.
NeRF-based approaches mitigate this ambiguity using semantic consistency
\citep{jain2021dietnerf}, geometry and appearance regularization at unseen views
\citep{niemeyer2022regnerf}, frequency and occlusion regularization
\citep{yang2023freenerf}, or depth-ranking priors from imperfect depth estimates
\citep{wang2023sparsenerf}. Similar challenges arise in 3DGS: under sparse supervision,
adaptive density control can overfit observed views, introduce floaters, and produce
background collapse or oversmoothed structures. Existing sparse-view 3DGS methods
therefore introduce depth priors, explicit consistency constraints, Gaussian unpooling,
or coherent structured parameterizations to stabilize geometry
\citep{xiong2023sparsegs,chung2024depthregularized,li2024dngaussian,
zhu2024fsgs,paliwal2024coherentgs}. 

\textbf{Natural Image Statistics and Wavelet Priors}.
Natural images exhibit well-characterized statistical regularities in the wavelet domain: sub-band coefficients follow sparse, heavy-tailed (high-kurtosis) distributions, and this structure is consistent across sub-bands and spatial frequencies \cite{field1987relations,ruderman1994statistics,simoncelli2001natural}. Wavelets provide a multi-resolution representation that separates image content across scale and orientation \cite{mallat1989theory,daubechies1992ten}. These properties underlie classical image processing techniques including sparse coding, compression, and denoising \cite{simoncelli1996noise,chang2000adaptive}. The DiffNat \cite{diffnat} paper demonstrated that kurtosis concentration across wavelet sub-bands is a useful self-supervised signal for diffusion model artifact removal. We build on this observation but show that supervised alignment against ground-truth statistics is required when the target domain (3DGS) produces qualitatively different artifacts such as floaters, streaking, and oversmoothing. More related works are provided in the supplementary \ref{suppl_sec:realted-work}.

\section{Method}
\label{sec:method}
\vspace{-0.3cm}
\subsection{Preliminaries: 3D Gaussian Splatting.}
\vspace{-0.3cm}
A 3DGS scene is parameterized by a set of $N$ Gaussians, each defined by a mean position $\boldsymbol{\mu} \in \mathbb{R}^3$, a covariance matrix $\boldsymbol{\Sigma} = \mathbf{R}\mathbf{S}\mathbf{S}^\top\mathbf{R}^\top$
(decomposed into rotation $\mathbf{R}$ and scale $\mathbf{S}$), an opacity $\alpha \in [0,1]$, and spherical harmonic coefficients encoding view-dependent color \cite{kerbl2023gaussians}. Training minimizes a combination of L1 and SSIM
losses against ground-truth images, with ADC applied periodically \cite{kerbl2023gaussians,wang2004image}.
\vspace{-0.3cm}
\subsection{Kurtosis Concentration Loss.}
\vspace{-0.3cm}
While the cross-band covariance penalty reduces redundancy between frequency components, it does not regulate how kurtosis is distributed across wavelet bands. During 3DGS optimization, a few bands may become overly dominant while others remain weak, leading to unstable refinement and view-inconsistent structure. To mitigate this, we introduce a kurtosis concentration loss that penalizes excessive spread in band-wise kurtosis, encouraging balanced yet heavy-tailed frequency responses.
Given flattened, standardized band responses $\hat{\mathbf{z}}_b(x) \in \mathbb{R}^{P}$, the kurtosis of band $b$ is,
\vspace{-0.2cm}
\begin{equation}
\kappa_b(x)
=
\frac{1}{P}
\sum_{p=1}^{P}
\left(\hat{\mathbf{z}}_b^{(p)}(x)\right)^4
- 3 .
\end{equation}
\vspace{-0.2cm}
The minibatch kurtosis concentration loss is then defined as,
\begin{equation}
\mathcal{L}_{\mathrm{kurt}}
=
\frac{1}{N}
\sum_{n=1}^{N}
\left[
\max_b \kappa_b(x_n)
-
\min_b \kappa_b(x_n)
\right].
\end{equation}
Minimizing this loss prevents any single wavelet band from dominating while preserving heavy-tailed spatial statistics, yielding more stable and structurally consistent 3DGS reconstructions.
\begin{figure}[!t]
    \centering
    \includegraphics[width=1.0\textwidth]{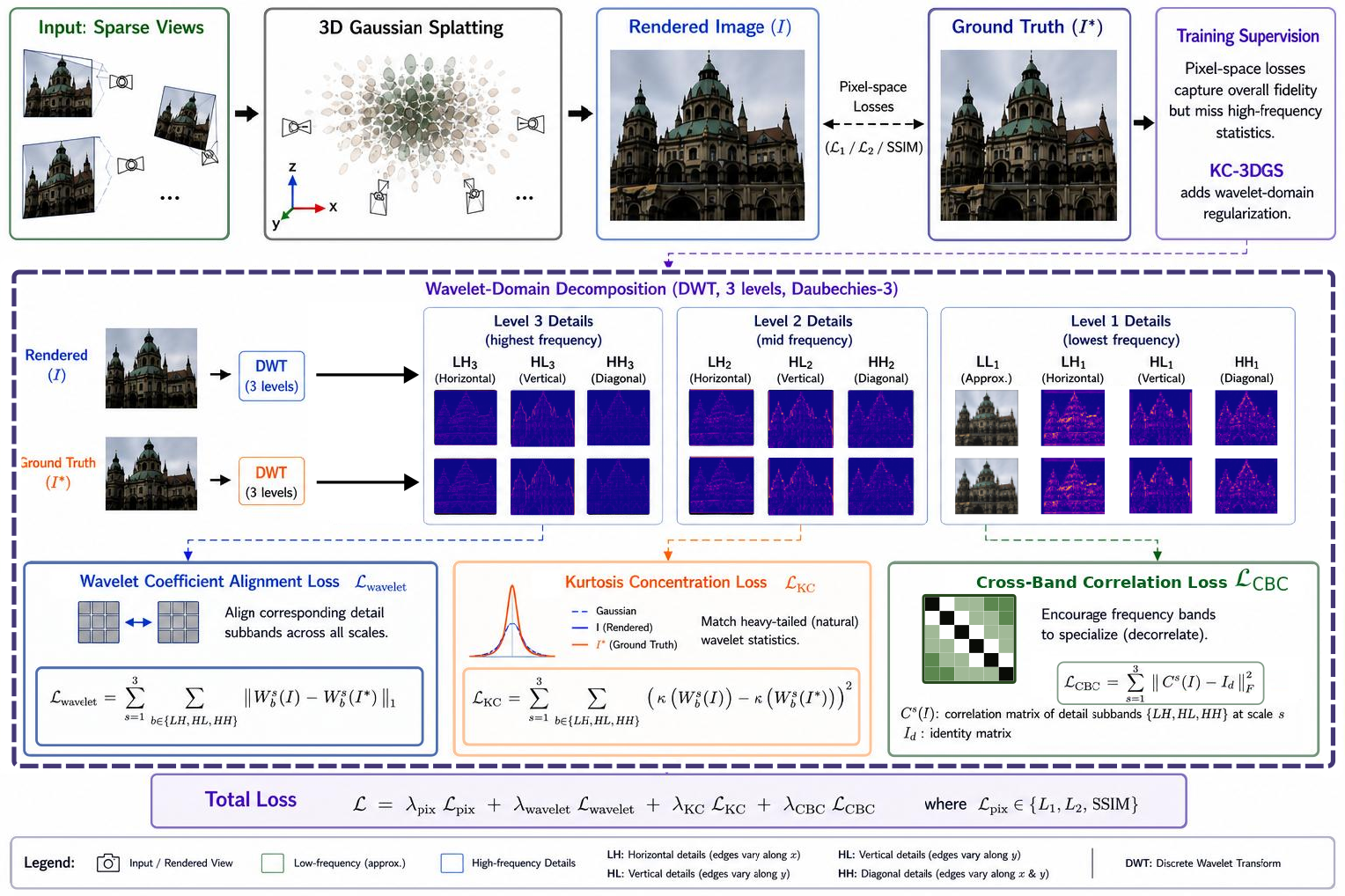}
    \caption{\textbf{KC-3DGS training pipeline.} Rendered and ground-truth images undergo 3-level Daubechies-3 wavelet decomposition to extract detail subbands. Three losses operate on these subbands: scale-weighted wavelet alignment ($\mathcal{L}_{wavelet}$), supervised kurtosis concentration ($\mathcal{L}_{KC}$) for matching heavy-tailed statistics, and cross-band correlation penalty ($\mathcal{L}_{CBC}$) for frequency specialization. Combined with standard $L_1$ and SSIM losses, gradients update Gaussian parameters to improve both distortion and perceptual metrics.} 
    \label{fig:model}
    \vspace{-0.4cm}
\end{figure}
\vspace{-0.3cm}
\subsection{Spatial Cross-Band Covariance Penalty.}
\vspace{-0.3cm}
In 3DGS, wavelet subbands should capture complementary structure: low-frequency
bands model coarse geometry and illumination, while high-frequency bands encode
textures and sharp details. However, redundant gradients across bands can cause
cross-band interference, oversmoothing, texture leakage, and unstable multi-view
refinement. We therefore introduce a spatial cross-band covariance regularizer
that encourages frequency specialization by suppressing statistical dependency
between subbands.
Let $z_b(x)\in\mathbb{R}^{H\times W}$ be the $b$-th wavelet subband of a
rendered image $x$, with $B$ total bands and $P=HW$. We flatten and center each
subband, stack the centered responses, and penalize off-diagonal covariance:
\begin{equation}
\begin{aligned}
\mathbf{z}_b(x) &\in \mathbb{R}^{P}, \qquad
\tilde{\mathbf{z}}_b(x)
=
\mathbf{z}_b(x)
-
\frac{1}{P}\sum_{p=1}^{P}\mathbf{z}_b^{(p)}(x), \\
\tilde{Z}(x)
&=
\big[
\tilde{\mathbf{z}}_1(x),\ldots,\tilde{\mathbf{z}}_B(x)
\big]^\top
\in \mathbb{R}^{B\times P}, \qquad
\Sigma_z(x)
=
\frac{1}{P-1}\tilde{Z}(x)\tilde{Z}(x)^\top, \\
\mathcal{L}_{\mathrm{cov}}
&=
\frac{1}{N}
\sum_{n=1}^{N}
\left\|
\Sigma_z(x_n)
-
\operatorname{Diag}\!\left(\operatorname{diag}(\Sigma_z(x_n))\right)
\right\|_F^2 .
\end{aligned}
\end{equation}
This loss encourages wavelet bands to encode distinct structural cues, improving
3DGS optimization stability while preserving high-frequency detail and coherent
low-frequency structure.
\vspace{-0.2cm}
\subsection{Wavelet Decomposition.}
\vspace{-0.2cm}
Given a rendered image $\mathbf{I}_\text{pred} \in \mathbb{R}^{H \times W\times 3}$ and corresponding ground truth $\mathbf{I}_\text{gt}$, we apply a $J$-level 2D discrete wavelet transform (DWT \cite{Walnut2004}) using the Daubechies-3 (db3) wavelet via a differentiable PyTorch implementation \cite{cotter_2019}. As illustrated in Figure~\ref{fig:wavelet_analysis}, the DWT recursively decomposes each channel into a low-frequency approximation subband (LL) and three detail subbands capturing horizontal (LH), vertical (HL), and diagonal (HH) edges. At each decomposition level, the approximation subband is further decomposed, yielding $J$ (we use 3) sets of detail subbands $\{c_H^j, c_V^j, c_D^j\}_{j=1}^J$ at progressively coarser scales, plus a final approximation $c_A$. The resulting wavelet tree (Figure~\ref{fig:wavelet_analysis}) reveals how edge information is distributed across scales: fine-scale subbands (bottom) capture high-frequency texture and sharp edges, while coarse-scale subbands (top-left quadrant) encode broader structural features. 

We directly align the wavelet coefficients of the rendered image to those of the ground truth via an L1 loss on each detail subband, with a per-level weight that increases with decomposition depth (finer scales weighted more heavily):
\begin{equation}
    \mathcal{L}_\text{wav} = \sum_{j=1}^{J} 2^j \left(
        \|c_H^{j,\text{pred}} - c_H^{j,\text{gt}}\|_1 +
        \|c_V^{j,\text{pred}} - c_V^{j,\text{gt}}\|_1 +
        \|c_D^{j,\text{pred}} - c_D^{j,\text{gt}}\|_1
    \right),
\end{equation}
where $j = 1$ corresponds to the finest scale. This weighting scheme addresses the energy decay inherent in natural images (Proposition~2), ensuring that fine-scale errors receive sufficient gradient signal to counteract the oversmoothing tendency of pixel-space losses.

\begin{figure}[t!]
    \centering
    \includegraphics[width=1.0\textwidth]{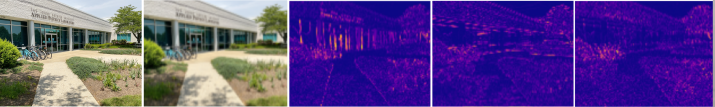}
    \caption{\textbf{Wavelet decomposition of a natural image.} Three-level discrete wavelet transform showing the ground truth, low-frequency approximation, and detail coefficients capturing vertical (HL), horizontal (LH), and diagonal (HH) edges at the coarsest resolution (in order left to right). We stack three resolution scales that get progressively finer. The kurtosis concentration loss exploits the statistical regularity of these band-pass subbands, where natural images exhibit characteristic heavy-tailed distributions that our method encourages in rendered outputs.} 
    \label{fig:wavelet_analysis}
    \vspace{-0.4cm}
\end{figure}

\vspace{-0.2cm}
\subsection{Supervised Wavelet Statistics Losses.}
\vspace{-0.2cm}
While the KC and cross-band correlation (CBC) losses described above can be applied in an unsupervised manner to encourage desirable statistical properties in rendered images, we found that unsupervised application leads to suboptimal results for 3DGS. The KC loss, which minimizes the range of kurtosis values across wavelet subbands, can be trivially satisfied by producing uniformly blurry outputs where all subbands exhibit similarly low kurtosis. Similarly, the CBC penalty, which discourages correlation between oriented subbands (LH, HL, HH) at the same spatial locations, may inadvertently penalize sharp corners and edge intersections which are high-frequency geometric details we wish to preserve. Without grounding these statistics to reference values, the optimization has no target for what constitutes ``correct'' frequency behavior.

To address this, we reformulate both losses as supervised alignment objectives. Given a rendered image $\mathbf{I}_\text{pred}$ and ground truth $\mathbf{I}_\text{gt}$, we compute the KC and CBC statistics for both images and penalize their absolute difference:
\begin{equation}
    \mathcal{L}_\text{KC}^\text{sup} = \left| L_\text{kurt}(\mathbf{I}_\text{pred}) - L_\text{kurt}(\mathbf{I}_\text{gt}) \right| + \sum_{b=1}^{B} \left| \kappa_b(\mathbf{I}_\text{pred}) - \kappa_b(\mathbf{I}_\text{gt}) \right|,
\end{equation}
\begin{equation}
    \mathcal{L}_\text{CBC}^\text{sup} = \left| L_\text{cov}(\mathbf{I}_\text{pred}) - L_\text{cov}(\mathbf{I}_\text{gt}) \right|.
\end{equation}
The supervised KC loss encourages the prediction to match both the overall kurtosis spread and the per-subband kurtosis values of the ground truth, ensuring that the rendered image exhibits the same frequency-dependent peakedness characteristics as natural images. The supervised CBC loss ensures that the spatial correlation structure between oriented subbands in the prediction matches that of the ground truth—if the ground truth contains corners that induce cross-band correlation, the prediction is encouraged to reproduce them rather than smooth them away. Combined with the direct wavelet coefficient alignment loss $\mathcal{L}_\text{wav}$, these supervised statistics losses provide multi-scale frequency guidance that complements pixel-space reconstruction objectives, enabling sharper and more geometrically faithful 3DGS reconstructions.
\vspace{-0.2cm}
\section{Theoretical Motivation}
\label{sec:theory}
\vspace{-0.2cm}
We formalize the intuition behind KC-3DGS through results that justify each
component of our objective. Let $I_{\text{pred}}, I_{\text{gt}} \in
\mathbb{R}^{H \times W}$ denote a single rendered channel and its ground
truth, and let $\mathcal{W}$ denote the $J$-level orthogonal 2D DWT producing
one approximation subband and $3J$ detail subbands $\{c^j_H, c^j_V,
c^j_D\}_{j=1}^J$. We assume an orthogonal wavelet basis (Daubechies in our
experiments), treating boundary effects as negligible.

\paragraph{The blind spot of pixel-space losses.}
Standard 3DGS objectives admit a continuous family of distinct renderings
with identical pixel-space loss energy.

\begin{proposition}[Pixel-space invariance under wavelet redistribution]
\label{prop:redistribution}
Let $\Delta = I_{\text{pred}} - I_{\text{gt}}$. Parseval's identity gives
$\|\Delta\|_2 = \|\mathcal{W}\Delta\|_2$, so per-subband energies may be
redistributed freely subject only to $\sum_{j} \|(\mathcal{W}\Delta)^j\|_2^2 =
\|\Delta\|_2^2$, yielding a $(3J{-}1)$-dimensional family of perturbations
indistinguishable under pixel-space $L^2$ supervision.
\end{proposition}

Proofs are provided in the supplementary \ref{sec:theory}. This formalizes that $L^2$ (and, modulo norm equivalence, $L^1$ and SSIM)
constrains only \emph{aggregate} error energy, not its \emph{distribution
across scales}. Subband supervision closes this loophole.

\begin{lemma}[Subband supervision strictly refines pixel supervision]
\label{lem:strict-refine}
For any orthogonal $\mathcal{W}$, $\sum_j \|(\mathcal{W}\Delta)^j\|_1 \geq
\|\Delta\|_1$, with equality iff $\Delta$ has single-subband support. Hence
$\mathcal{L}_{\text{wav}} = 0 \implies \mathcal{L}_1 = 0$, but not conversely.
\end{lemma}


\paragraph{Why finer scales must be up-weighted.}
For natural images modeled as $1/f^{\alpha}$ processes with $\alpha \in [1,2]$, the expected per-coefficient energy at level $j$ scales as $2^{-j(\alpha-1)}\!\cdot\!2^{-2j}$, accounting for the $4^j$-fold reduction in coefficient count. An unweighted $\sum_j \|(\mathcal{W}\Delta)^j\|_1$ is therefore dominated by coarse-scale terms (precisely the scales 3DGS already handles well), while fine-scale errors that produce oversmoothing contribute negligibly. Re-weighting by $2^j$ inverts this falloff and approximately equalizes optimization influence across scales, giving $\mathcal{L}_{\text{wav}}=\sum_{j=1}^J 2^j \sum_{o\in\{H,V,D\}} \|c^{j,\text{pred}}_o-c^{j,\text{gt}}_o\|_1$.

\paragraph{Why kurtosis concentration needs an anchor.}
Define the excess kurtosis of subband $b$ as $\kappa_b =
\mathbb{E}[\hat{z}_b^4] - 3$. Natural-image subbands are heavy-tailed with
$\kappa_b^{\text{gt}}$ in a narrow positive range. A naive minimization of the
spread $\kappa_{\max} - \kappa_{\min}$ is degenerate:

\begin{theorem}[Degeneracy of unsupervised KC]
\label{thm:degeneracy}
The objective $\mathcal{L}^{\mathrm{unsup}}_{\text{kurt}}(I) = \kappa_{\max}(I)
- \kappa_{\min}(I)$ is globally minimized by any image with jointly Gaussian
detail subbands --- including the constant image. For any $I_{\text{gt}}$
with $\kappa_*^{\text{gt}} > 0$, the unsupervised optimum is generically not
the ground truth.
\end{theorem}

The remedy is to combine KC with a supervisory term that anchors per-band
statistics:

\begin{theorem}[Joint identifiability]
\label{thm:identifiability}
Let $\mathcal{L}_{\text{joint}} = \mathcal{L}_{\text{wav}} + \lambda
\mathcal{L}^{\mathrm{unsup}}_{\text{kurt}}$ with $\lambda > 0$, and suppose
$\min_b \kappa_b^{\text{gt}} > 0$. Any global minimizer $I^*$ satisfies
$\kappa_b(I^*) > 0$ for every $b$, excluding the Gaussianized minimizer of
Theorem~\ref{thm:degeneracy}.
\end{theorem}

The argument is that if $\kappa_b(I^*) = 0$, the standardized fourth moments
of $c_b^*$ and $c_b^{\text{gt}}$ differ by at least $\kappa_b^{\text{gt}}$,
forcing $\|c_b^* - c_b^{\text{gt}}\|_1 \geq C \kappa_b^{\text{gt}} > 0$ by a
moment-matching bound on bounded-support distributions. Since $I_{\text{gt}}$
attains zero on both terms, $I^*$ cannot be optimal. The two losses are
individually under-determined but jointly well-posed:
$\mathcal{L}_{\text{wav}}$ supplies the supervisory anchor, while
$\mathcal{L}^{\mathrm{unsup}}_{\text{kurt}}$ adds the higher-order shape
constraint that $L^1$ does not directly enforce.

\paragraph{Cross-band covariance as decorrelation.}
Modeling natural-image subbands as approximately factorized,
$p(z_1, \ldots, z_B) = \prod_b p_b(z_b)$, motivates the third regularizer.

\begin{lemma}[Frobenius decorrelation as Gaussian-factorization KL]
\label{lem:cov-kl}
Let $\Sigma_z$ be the empirical cross-band covariance. The squared
off-diagonal Frobenius norm $\|\Sigma_z - \mathrm{diag}(\Sigma_z)\|_F^2$
equals, to leading order, the KL divergence between $\mathcal{N}(0, \Sigma_z)$
and the closest factorized Gaussian $\prod_b \mathcal{N}(0, \Sigma_{bb})$.
With unit marginal variances,
$\mathrm{KL} = \tfrac{1}{2} \|\Sigma_z - \mathrm{diag}(\Sigma_z)\|_F^2 +
O(\|E\|^3)$, where $E$ collects the off-diagonals.
\end{lemma}

This grounds $\mathcal{L}_{\text{cov}}$ as a tractable likelihood-based
decorrelation criterion that is genuinely \emph{unsupervised}: it encodes a
property of the target distribution (subband independence) rather than of any
specific reference image.

The three regularizers play distinct, non-redundant roles dictated by the
preceding results. $\mathcal{L}_{\text{wav}}$ supplies scale-resolved
supervision (Lemma~\ref{lem:strict-refine}) with finer scales up-weighted by
$2^j$ to counter the $1/f^\alpha$ energy falloff;
$\mathcal{L}^{\mathrm{unsup}}_{\text{kurt}}$ enforces higher-order shape but
is degenerate alone (Theorem~\ref{thm:degeneracy}) and well-posed only when
paired with $\mathcal{L}_{\text{wav}}$ (Theorem~\ref{thm:identifiability});
and $\mathcal{L}_{\text{cov}}$ enforces band specialization as an unsupervised
decorrelation prior (Lemma~\ref{lem:cov-kl}). Removing any one term
reproduces the failure mode that the corresponding theoretical statement
predicts.
\section{Experiments}
\label{sec:experiments}
\begin{table}[t!]
  \centering
    \caption{Sparse-view setup on 3DGS with Splatfacto}
  \scalebox{0.8}{
  \begin{tabular}{llcccc}
    \toprule
    \textbf{Dataset} & \textbf{Method} & \textbf{PSNR} $\uparrow$ & \textbf{SSIM} $\uparrow$ & \textbf{LPIPS} $\downarrow$ & \textbf{DreamSim} $\downarrow$\\
    \midrule
    MipNeRF360~\cite{mipnerf360} & Baseline & 16.35 & 0.3680 & \textbf{0.4218} & 0.2124\\
     & KC-3DGS & \textbf{16.88} & \textbf{0.3886} & 0.4281 & \textbf{0.2090} (\textcolor{green}{+1.60\%})\\
    \midrule
    Tanks\&Temples~\cite{Knapitsch2017} & Baseline & \textbf{22.33} & \textbf{0.6138} & \textbf{0.0961} & 0.0138\\
     & KC-3DGS & 22.27 & 0.6109 & 0.0974 & \textbf{0.0137} (\textcolor{green}{+0.7\%})\\
    \midrule
    MVImgNet~\cite{yu2023mvimgnet} & Baseline & 21.92 & 0.5909 & 0.1432 & 0.0208\\
     & KC-3DGS & \textbf{22.23} & \textbf{0.5938} & \textbf{0.1367} & \textbf{0.0194} (\textcolor{green}{+6.73\%})\\
    \bottomrule
  \end{tabular}
  }
  \label{tab:metrics}
\end{table}


\vspace{-0.3cm}
\subsection{Datasets and Evaluation Protocol.}
\vspace{-0.2cm}
We evaluate on five benchmarks spanning structured academic scenes and realistic sparse-view site modeling. 
\textbf{MipNeRF360}~\cite{mipnerf360} covers diverse indoor and outdoor scenes; 
\textbf{Tanks\&Temples}~\cite{Knapitsch2017} tests large-scale geometry with wide-baseline views; 
\textbf{MVImgNet}~\cite{yu2023mvimgnet} evaluates object-centric multi-view generalization; and 
\textbf{DeepBlending}~\cite{DeepBlending2018} contains bounded indoor scenes with challenging lighting and reflective surfaces.

Our most challenging setting is \textbf{WRIVA-ULTRRA}~\cite{ultra}, the ULTRRA-formatted WRIVA challenge dataset. 
Unlike conventional benchmarks with relatively structured captures, WRIVA-ULTRRA reflects realistic site-modeling conditions: sparse and noisy image coverage, transient objects, illumination variation, heterogeneous camera models, varying resolutions, and irregular viewpoints. 
These factors make sparse-view 3DGS highly underconstrained, often producing weakly supported Gaussians, floaters, streak artifacts, oversmoothing, and inconsistent geometry. 
We therefore use WRIVA-ULTRRA as a stress test for robustness under real-world sparse capture, evaluating the A01 site across increasingly limited view settings.

For sparse-view experiments on MipNeRF360, Tanks\&Temples, and MVImgNet, we use 12-view subsets with MASt3R~\cite{mast3r_eccv24} poses released by SPARS3R~\cite{tang2024spars3r}. 
WRIVA-ULTRRA uses predefined sparse splits with 10, 15, and 50 views, while DeepBlending is evaluated only in the full-scene setting due to unavailable sparse registrations. 
Initial ablations are implemented in Nerfstudio's Splatfacto~\cite{xu2024splatfactow}, and larger full-scene experiments use NerfICG's FasterGS~\cite{hahlbohm2026fastergs}. 
We report PSNR, SSIM, LPIPS (VGG), and DreamSim.
\begin{table}
\centering
\caption{Comparison on full-view evaluations with FasterGS}
\scalebox{0.8}{
\begin{tabular}{llcccc}
\toprule
\textbf{Dataset} & \textbf{Method} & \textbf{PSNR} $\uparrow$ & \textbf{SSIM} $\uparrow$ & \textbf{LPIPS} $\downarrow$ & \textbf{DreamSim} $\downarrow$ \\
\midrule
DeepBlending~\cite{DeepBlending2018} & Baseline & \textbf{30.0212} & \textbf{0.9214} & \textbf{0.0434} & 0.0336\\
 & KC-3DGS & 30.0146 & 0.9203 & 0.0440 & \textbf{0.0330} (\textcolor{green}{+1.79\%}) \\
\midrule
MipNeRF360~\cite{mipnerf360} & Baseline & \textbf{27.8870} & \textbf{0.8238} & \textbf{0.1065} & 0.0172 \\
 & KC-3DGS & 26.9394 & 0.7748 & 0.1335 & \textbf{0.0157} (\textcolor{green}{+8.72\%}) \\
\midrule
Tanks\&Temples~\cite{Knapitsch2017} & Baseline & \textbf{19.6339} & \textbf{0.6571} & 0.2340 & 0.0967 \\
 & KC-3DGS & 19.5508 & 0.6560 & \textbf{0.2279} & \textbf{0.0961} (\textcolor{green}{+0.62\%}) \\
\midrule
WRIVA-ULTRRA~\cite{ultra} & Baseline & 15.6250 & 0.4030 & 0.4385 & 0.3081 \\
 & KC-3DGS & \textbf{16.2693} & \textbf{0.4176} & \textbf{0.4217} & \textbf{0.2789} (\textcolor{green}{+9.48\%}) \\
\bottomrule
\end{tabular}}
\vspace{-0.5cm}
\label{tab:full-scene}
\end{table}


\vspace{-0.2cm}
 \begin{figure}[!t]
    \centering
    \includegraphics[width=0.8\textwidth]{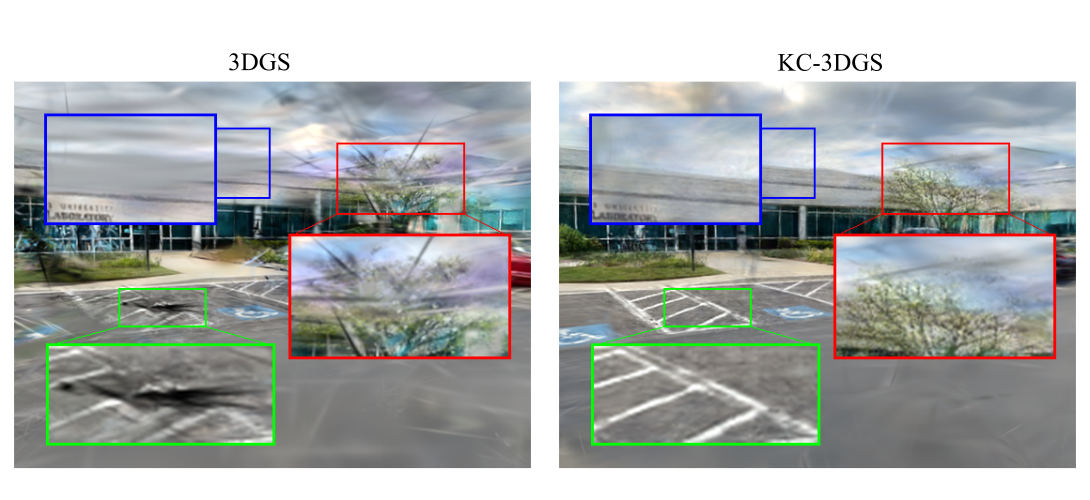}
    \vspace{-0.5cm}
    \caption{Qualitative comparison of FasterGS on the WRIVA-ULTRRA dataset, a challenging outdoor scene with sparse viewpoints. \textcolor{blue}{\textbf{Blue}}: Our method produces cleaner building facade reconstruction with reduced streak artifacts and improved window grid definition. \textcolor{red}{\textbf{Red}}: Foliage detail is better preserved, with sharper leaf structure and reduced color bleeding compared to the baseline's blurred appearance. \textcolor{green}{\textbf{Green}}: Asphalt texture and parking lot markings show improved clarity; the baseline exhibits characteristic over-smoothing on ground plane textures.}
    \label{fig:wriva_zoomed}
    \vspace{-0.4cm}
\end{figure}
\subsection{Main Results}
\vspace{-0.2cm}
\paragraph{Sparse-View Evaluation.}
Table~\ref{tab:metrics} shows that KC-3DGS improves sparse-view reconstruction when only 12 training images are available per scene. On MipNeRF360, the full objective increases PSNR and SSIM while lowering DreamSim, indicating improved perceptual fidelity without degrading distortion-based metrics. On MVImgNet, KC-3DGS yields the strongest gains, reducing DreamSim by 6.73\% while also improving PSNR and LPIPS. This suggests that wavelet-domain regularization is particularly effective for object-centric sparse-view reconstruction, where multi-scale frequency constraints help recover fine geometric and textural details.
On Tanks\&Temples, the gains are smaller and the method performs comparably to the baseline. We attribute this to the SPARS3R train/test splits, which contain substantial viewpoint overlap and therefore leave limited novel-view extrapolation to evaluate. This ceiling effect motivates evaluation on more challenging sparse-view benchmarks such as WRIVA-ULTRRA, where irregular camera distributions and limited coverage better expose 3DGS failure modes. Extensive ablations are in the supplementary \ref{tab:sup_metrics_sppl}.
\begin{wraptable}{r}{0.5\textwidth}
  \centering
  \vspace{-0.4cm}
  \caption{Comparison on the WRIVA-ULTRRA dataset with OctreeGS \cite{ren2024octree}.}
  \scalebox{0.85}{
  \begin{tabular}{llc}
    \toprule
    \textbf{Total Images} & \textbf{Method} & \textbf{DreamSim} $\downarrow$ \\
    \midrule
    50 & OctreeGS & 0.1751 \\
       & KC-3DGS  & \textbf{0.1670} (\textcolor{green}{+4.63\%}) \\
    \midrule
    15 & OctreeGS & 0.1660 \\
       & KC-3DGS  & \textbf{0.1622} (\textcolor{green}{+2.29\%}) \\
    \midrule
    10 & OctreeGS & 0.1744 \\
       & KC-3DGS  & \textbf{0.1628} (\textcolor{green}{+6.65\%}) \\
    \bottomrule
  \end{tabular}}
  \vspace{-0.5cm}
  \label{tab:varying-views}
\end{wraptable}
\vspace{-0.2cm}
\paragraph{Full-Scene Evaluation.} Table~\ref{tab:full-scene} reports full-scene results using FasterGS. Under dense-view supervision, KC-3DGS primarily improves perceptual quality, achieving consistent DreamSim reductions across datasets. The largest gain occurs on WRIVA-ULTRRA, where KC-3DGS improves PSNR, SSIM, and DreamSim, with a 9.48\% relative DreamSim reduction. This highlights the benefit of kurtosis-based regularization for challenging real-world outdoor scenes with noisy, heterogeneous capture conditions.
On MipNeRF360, KC-3DGS reduces DreamSim by 8.72\% despite slight decreases in PSNR and SSIM, suggesting a trade-off between pixel-level reconstruction accuracy and perceptual fidelity. DeepBlending and Tanks\&Temples show smaller but consistent improvements, indicating that the method remains useful in well-constrained dense-view settings, although its benefits are most pronounced in difficult real-world scenarios where standard 3DGS is less stable.
\vspace{-0.2cm}
\paragraph{Varying Training Views.} We further evaluate WRIVA-ULTRRA with 10, 15, and 50 input images to study the effect of training-set size (Table~\ref{tab:varying-views}). KC-3DGS achieves relative DreamSim reductions of 6.6\%, 2.3\%, and 4.6\% for the 10-, 15-, and 50-view settings, respectively. These results show that wavelet-domain regularization is especially helpful when pixel-space supervision is limited, providing complementary frequency-domain constraints that improve perceptual quality in sparse-view reconstruction.
\begin{figure}[!t]
    \centering
    \includegraphics[width=0.8\textwidth]{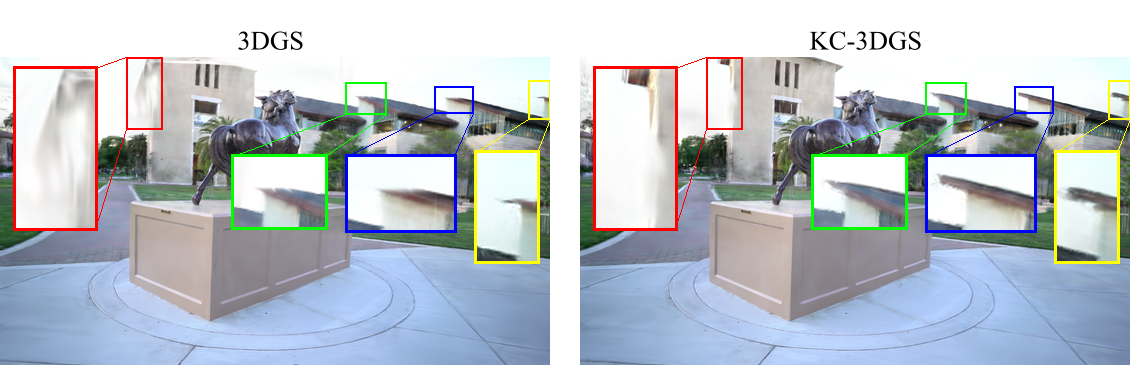}
    \vspace{-0.2cm}
    \caption{Qualitative comparison between baseline 3DGS (left) and our method (right) on a scene from Tanks\&Temples. Zoomed regions highlight key differences: \textcolor{red}{\textbf{Red}}: Our method reduces floater artifacts near the building edge, producing cleaner geometry. \textcolor{green}{\textbf{Green}}, \textcolor{blue}{\textbf{Blue}}, and \textcolor{yellow}{\textbf{Yellow}}: Reduced over-smoothing on the rooftops, retaining finer structure.}
    \label{fig:horse_zoomed}
    \vspace{-0.5cm}
\end{figure}
\vspace{-0.2cm}
\paragraph{Qualitative Analysis.} Figures~\ref{fig:wriva_zoomed} and~\ref{fig:horse_zoomed} show the qualitative gains from our wavelet-based regularization across diverse scene types. In the challenging WRIVA-ULTRRA outdoor scene (Figure~\ref{fig:wriva_zoomed}), our method produces cleaner building facades with fewer streak artifacts and sharper windows, preserves finer foliage structure with reduced color bleeding, and recovers clearer ground-plane textures such as asphalt details and parking-lot markings.
The Tanks\&Temples example (Figure~\ref{fig:horse_zoomed}) shows similar benefits in a structured outdoor scene. Our method reduces floater artifacts near building boundaries, where sparse-view supervision leaves baseline Gaussians weakly constrained, and preserves rooftop details that are oversmoothed by the baseline. These improvements support our motivation: wavelet supervision strengthens high-frequency gradients that pixel-space losses often underweight, while the KC and cross-band correlation regularizers stabilize frequency responses and suppress artifacts. The gains are most visible in regions with fine texture and limited view coverage. Further qualitative analysis is in the supplementary \ref{sec:sup_qualitative}.


\begin{figure}
    \centering
    \includegraphics[width=0.8\textwidth]{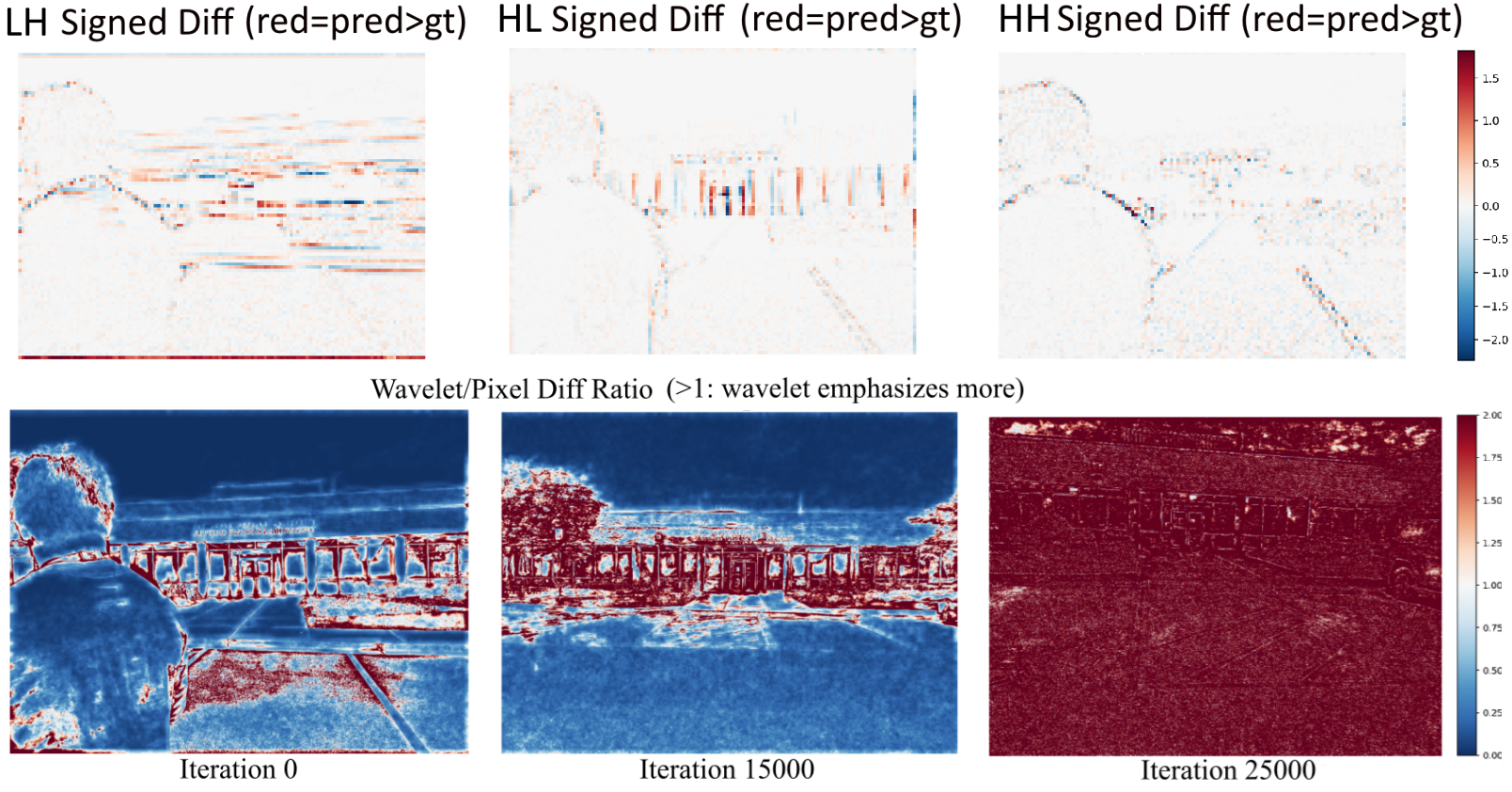}
    \vspace{-0.2cm}
    \caption{\textbf{L3 wavelet subbands over training.} First row: Signed differences between coarse-scale wavelet coefficients of predicted and ground-truth images across LH, HL, and HH detail subbands. Second row: Ratio of normalized wavelet detail differences to normalized pixel $L_1$ differences (\textcolor{red}{red}: wavelet-emphasized regions, \textcolor{blue}{blue}: pixel-emphasized regions) on WRIVA-ULTRRA with 50 training images. Over training, wavelet-domain errors increasingly dominate pixel-space errors, especially in fine structural regions. Further training ablations and analysis are in the supplementary \ref{tab:sup_ablation} \ref{fig:sup_loss_plot_errors} \ref{tab:sup_train_time}.}
    \label{fig:wavelet_progression}
    \vspace{-0.5cm}
\end{figure}
\vspace{-0.2cm}

\paragraph{Interpretability Analysis.}
Figure~\ref{fig:wavelet_progression} shows that wavelet-domain supervision emphasizes errors underweighted by pixel losses. Early in training, rendered LH/HL/HH subbands are weaker than the ground truth, indicating missing horizontal, vertical, and diagonal high-frequency structure. Wavelet $L_1$ errors concentrate on structural regions such as sidewalks, building edges, foliage, and parking markings rather than homogeneous sky. The normalized wavelet-to-pixel error ratio further shows a shift from pixel-dominated coarse errors early in training to wavelet-dominated fine-structure errors later, explaining the improved perceptual quality and DreamSim gains.


\vspace{-0.2cm}
\paragraph{Limitations.}
KC-3DGS is most effective in sparse-view settings, where frequency-domain constraints provide useful supervision beyond pixel losses. With dense multi-view coverage, however, these constraints can over-regularize optimization, as seen in the full-view MipNeRF360 results in Table~\ref{tab:full-scene}. KC-3DGS also adds computational overhead from the wavelet-domain losses, which we try to mitigate in future.

\vspace{-0.2cm}
\section{Conclusion}
\vspace{-0.2cm}
We introduced KC-3DGS, a high-fidelity novel view synthesis framework that augments 3DGS optimization with wavelet-domain kurtosis concentration constraints. By regularizing reconstruction behavior across frequency scales, KC-3DGS addresses a limitation of standard pixel-space objectives, which may achieve strong aggregate scores while still producing oversmoothed textures, floaters, or structurally inconsistent views.
Experiments on MipNeRF360, Tanks\&Temples, MVImgNet, DeepBlending, and the challenging WRIVA-ULTRRA benchmark show that KC-3DGS consistently improves perceptual quality, with relative DreamSim reductions of up to 9.48\% on sparse outdoor scenes. The benefits are most pronounced in sparse-view settings, where wavelet-domain supervision provides complementary constraints beyond pixel losses. Qualitative results further show cleaner geometry, fewer artifacts, sharper textures, and improved structural coherence in fine-detail regions with limited view coverage.
\bibliographystyle{plain}
\bibliography{main}

@data{ultra,
doi = {10.21227/cjk5-gf33},
url = {https://dx.doi.org/10.21227/cjk5-gf33},
author = {Myron Brown and Michael Chan and Michael Twardowski},
publisher = {IEEE Dataport},
title = {WRIVA Public Data},
year = {2024} }

@ARTICLE{dngaussian,
  author={Li, Jiahe and Zhang, Jiawei and Yu, Xiaohan and Bai, Xiao and Zheng, Jin and Ning, Xin and Gu, Lin},
  journal={IEEE Transactions on Pattern Analysis and Machine Intelligence}, 
  title={DNGaussian++: Improving Sparse-View Gaussian Radiance Fields with Depth Normalization}, 
  year={2026},
  volume={},
  number={},
  pages={1-18},
  doi={10.1109/TPAMI.2026.3664307}}

@misc{fsgs, 
title={FSGS: Real-Time Few-Shot View Synthesis using Gaussian Splatting}, 
author={Zehao Zhu and Zhiwen Fan and Yifan Jiang and Zhangyang Wang}, 
year={2023},
eprint={2312.00451},
archivePrefix={arXiv},
primaryClass={cs.CV} 
}

@article{corgs,
      title={CoR-GS: Sparse-View 3D Gaussian Splatting via Co-Regularization},
      author={Zhang, Jiawei and Li, Jiahe and Yu, Xiaohan and Huang, Lei and Gu, Lin and Zheng, Jin and Bai, Xiao},
      journal={arXiv preprint arXiv:2405.12110},
      year={2024}
    }

@article{
diffnat,
title={DiffNat : Exploiting the Kurtosis Concentration Property for Image quality improvement},
author={Aniket Roy and Maitreya Suin and Anshul Shah and Ketul Shah and Jiang Liu and Rama Chellappa},
journal={Transactions on Machine Learning Research},
issn={2835-8856},
year={2025},
url={https://openreview.net/forum?id=HdZQ7pMPRd},
note={}
}

@INPROCEEDINGS {mipnerf360,
author = { Barron, Jonathan T. and Mildenhall, Ben and Verbin, Dor and Srinivasan, Pratul P. and Hedman, Peter },
booktitle = { 2022 IEEE/CVF Conference on Computer Vision and Pattern Recognition (CVPR) },
title = {{ Mip-NeRF 360: Unbounded Anti-Aliased Neural Radiance Fields }},
year = {2022},
volume = {},
ISSN = {},
pages = {5460-5469},
doi = {10.1109/CVPR52688.2022.00539},
url = {https://doi.ieeecomputersociety.org/10.1109/CVPR52688.2022.00539},
publisher = {IEEE Computer Society},
address = {Los Alamitos, CA, USA},
month =Jun}

@article{Knapitsch2017,
    author    = {Arno Knapitsch and Jaesik Park and Qian-Yi Zhou and Vladlen Koltun},
    title     = {Tanks and Temples: Benchmarking Large-Scale Scene Reconstruction},
    journal   = {ACM Transactions on Graphics},
    volume    = {36},
    number    = {4},
    year      = {2017},
}

@inproceedings{yu2023mvimgnet,
                        title     = {MVImgNet: A Large-scale Dataset of Multi-view Images},
                        author    = {Yu, Xianggang and Xu, Mutian and Zhang, Yidan and Liu, Haolin and Ye, Chongjie and Wu, Yushuang and Yan, Zizheng and Liang, Tianyou and Chen, Guanying and Cui, Shuguang and Han, Xiaoguang},
                        booktitle = {CVPR},
                        year      = {2023}
                    }

@article{DeepBlending2018,
  author = {Hedman, Peter and Philip, Julien and Price, True and Frahm, Jan-Michael and Drettakis, George and Brostow, Gabriel},
  title = {Deep Blending for Free-viewpoint Image-based Rendering},
  booktitle = {ACM Transactions on Graphics (Proc. SIGGRAPH Asia)},
  publisher = {ACM},
  volume    = {37},
  number    = {6},
  pages     = {257:1--257:15},
  year      = {2018}
}

@article{xu2024splatfactow,
  title={Splatfacto-W: A Nerfstudio Implementation of Gaussian Splatting for Unconstrained Photo Collections},
  author={Xu, Kevin and others},
  journal={arXiv preprint arXiv:2407.12306},
  year={2024}
}

@misc{hahlbohm2026fastergs,
  title         = {Faster-GS: Analyzing and Improving Gaussian Splatting Optimization},
  author        = {Florian Hahlbohm and Linus Franke and Martin Eisemann and Marcus Magnor},
  year          = {2026},
  eprint        = {2602.09999},
  archivePrefix = {arXiv},
  primaryClass  = {cs.CV},
  url           = {https://arxiv.org/abs/2602.09999},
}

@misc{mast3r_eccv24,
      title={Grounding Image Matching in 3D with MASt3R}, 
      author={Vincent Leroy and Yohann Cabon and Jerome Revaud},
      booktitle = {ECCV},
      year = {2024}
}

@article{tang2024spars3r,
  title={SPARS3R: Semantic Prior Alignment and Regularization for Sparse 3D Reconstruction},
  author={Tang, Yutao and Guo, Yuxiang and Li, Deming and Peng, Cheng},
  journal={arXiv preprint arXiv:2411.12592},
  year={2024}
}

@phdthesis{cotter_2019, title={Uses of Complex Wavelets in Deep Convolutional Neural Networks}, url={https://www.repository.cam.ac.uk/handle/1810/306661}, DOI={10.17863/CAM.53748}, school={Apollo - University of Cambridge Repository}, author={Cotter, Fergal}, year={2019} }

@article{ren2024octree,
  title={Octree-gs: Towards consistent real-time rendering with lod-structured 3d gaussians},
  author={Ren, Kerui and Jiang, Lihan and Lu, Tao and Yu, Mulin and Xu, Linning and Ni, Zhangkai and Dai, Bo},
  journal={arXiv preprint arXiv:2403.17898},
  year={2024}
}

@inproceedings{mildenhall2020nerf,
  title     = {{NeRF}: Representing Scenes as Neural Radiance Fields for View Synthesis},
  author    = {Mildenhall, Ben and Srinivasan, Pratul P. and Tancik, Matthew and Barron, Jonathan T. and Ramamoorthi, Ravi and Ng, Ren},
  booktitle = {European Conference on Computer Vision (ECCV)},
  pages     = {405--421},
  year      = {2020},
  publisher = {Springer}
}

@inproceedings{barron2022mipnerf360,
  title     = {{Mip-NeRF} 360: Unbounded Anti-Aliased Neural Radiance Fields},
  author    = {Barron, Jonathan T. and Mildenhall, Ben and Verbin, Dor and Srinivasan, Pratul P. and Hedman, Peter},
  booktitle = {Proceedings of the IEEE/CVF Conference on Computer Vision and Pattern Recognition (CVPR)},
  pages     = {5460--5469},
  year      = {2022}
}

@inproceedings{fridovichkeil2022plenoxels,
  title     = {Plenoxels: Radiance Fields Without Neural Networks},
  author    = {Fridovich-Keil, Sara and Yu, Alex and Tancik, Matthew and Chen, Qinhong and Recht, Benjamin and Kanazawa, Angjoo},
  booktitle = {Proceedings of the IEEE/CVF Conference on Computer Vision and Pattern Recognition (CVPR)},
  pages     = {5501--5510},
  year      = {2022}
}

@article{mueller2022instant,
  title     = {Instant Neural Graphics Primitives with a Multiresolution Hash Encoding},
  author    = {M{\"u}ller, Thomas and Evans, Alex and Schied, Christoph and Keller, Alexander},
  journal   = {ACM Transactions on Graphics},
  volume    = {41},
  number    = {4},
  pages     = {102:1--102:15},
  year      = {2022},
  doi       = {10.1145/3528223.3530127}
}

@article{kerbl2023gaussians,
  title     = {{3D} Gaussian Splatting for Real-Time Radiance Field Rendering},
  author    = {Kerbl, Bernhard and Kopanas, Georgios and Leimk{\"u}hler, Thomas and Drettakis, George},
  journal   = {ACM Transactions on Graphics},
  volume    = {42},
  number    = {4},
  pages     = {139:1--139:14},
  year      = {2023},
  doi       = {10.1145/3592433}
}

@inproceedings{yu2024mipsplatting,
  title     = {Mip-Splatting: Alias-Free {3D} Gaussian Splatting},
  author    = {Yu, Zehao and Chen, Anpei and Huang, Binbin and Sattler, Torsten and Geiger, Andreas},
  booktitle = {Proceedings of the IEEE/CVF Conference on Computer Vision and Pattern Recognition (CVPR)},
  year      = {2024}
}

@inproceedings{huang20242dgs,
  title     = {{2D} Gaussian Splatting for Geometrically Accurate Radiance Fields},
  author    = {Huang, Binbin and Yu, Zehao and Chen, Anpei and Geiger, Andreas and Gao, Shenghua},
  booktitle = {ACM SIGGRAPH 2024 Conference Papers},
  pages     = {1--11},
  year      = {2024},
  doi       = {10.1145/3641519.3657428}
}

@inproceedings{jain2021dietnerf,
  title     = {Putting {NeRF} on a Diet: Semantically Consistent Few-Shot View Synthesis},
  author    = {Jain, Ajay and Tancik, Matthew and Abbeel, Pieter},
  booktitle = {Proceedings of the IEEE/CVF International Conference on Computer Vision (ICCV)},
  pages     = {5885--5894},
  year      = {2021}
}

@inproceedings{niemeyer2022regnerf,
  title     = {{RegNeRF}: Regularizing Neural Radiance Fields for View Synthesis from Sparse Inputs},
  author    = {Niemeyer, Michael and Barron, Jonathan T. and Mildenhall, Ben and Sajjadi, Mehdi S. M. and Geiger, Andreas and Radwan, Noha},
  booktitle = {Proceedings of the IEEE/CVF Conference on Computer Vision and Pattern Recognition (CVPR)},
  pages     = {5470--5480},
  year      = {2022},
  doi       = {10.1109/CVPR52688.2022.00540}
}

@inproceedings{yang2023freenerf,
  title     = {{FreeNeRF}: Improving Few-Shot Neural Rendering with Free Frequency Regularization},
  author    = {Yang, Jiawei and Pavone, Marco and Wang, Yue},
  booktitle = {Proceedings of the IEEE/CVF Conference on Computer Vision and Pattern Recognition (CVPR)},
  year      = {2023}
}

@inproceedings{wang2023sparsenerf,
  title     = {{SparseNeRF}: Distilling Depth Ranking for Few-Shot Novel View Synthesis},
  author    = {Wang, Guangcong and Chen, Zhaoxi and Loy, Chen Change and Liu, Ziwei},
  booktitle = {Proceedings of the IEEE/CVF International Conference on Computer Vision (ICCV)},
  pages     = {9065--9076},
  year      = {2023}
}

@misc{xiong2023sparsegs,
  title         = {{SparseGS}: Real-Time 360{\textdegree} Sparse View Synthesis using Gaussian Splatting},
  author        = {Xiong, Haolin and Muttukuru, Sairisheek and Upadhyay, Rishi and Chari, Pradyumna and Kadambi, Achuta},
  year          = {2023},
  eprint        = {2312.00206},
  archivePrefix = {arXiv},
  primaryClass  = {cs.CV}
}

@inproceedings{chung2024depthregularized,
  title     = {Depth-Regularized Optimization for {3D} Gaussian Splatting in Few-Shot Images},
  author    = {Chung, Jaeyoung and Oh, Jeongtaek and Lee, Kyoung Mu},
  booktitle = {Proceedings of the IEEE/CVF Conference on Computer Vision and Pattern Recognition Workshops (CVPRW)},
  year      = {2024}
}

@inproceedings{li2024dngaussian,
  title     = {{DNGaussian}: Optimizing Sparse-View {3D} Gaussian Radiance Fields with Global-Local Depth Normalization},
  author    = {Li, Jiahe and Zhang, Jiawei and Bai, Xiao and Zheng, Jin and Ning, Xin and Zhou, Jun and Gu, Lin},
  booktitle = {Proceedings of the IEEE/CVF Conference on Computer Vision and Pattern Recognition (CVPR)},
  year      = {2024}
}

@inproceedings{zhu2024fsgs,
  title     = {{FSGS}: Real-Time Few-Shot View Synthesis using Gaussian Splatting},
  author    = {Zhu, Zehao and Fan, Zhiwen and Jiang, Yifan and Wang, Zhangyang},
  booktitle = {European Conference on Computer Vision (ECCV)},
  year      = {2024}
}

@inproceedings{paliwal2024coherentgs,
  title     = {{CoherentGS}: Sparse Novel View Synthesis with Coherent {3D} Gaussians},
  author    = {Paliwal, Avinash and Ye, Wei and Xiong, Jinhui and Kotovenko, Dmytro and Ranjan, Rakesh and Chandra, Vikas and Kalantari, Nima Khademi},
  booktitle = {European Conference on Computer Vision (ECCV)},
  pages     = {19--37},
  year      = {2024},
  publisher = {Springer}
}

@article{field1987relations,
  title   = {Relations Between the Statistics of Natural Images and the Response Properties of Cortical Cells},
  author  = {Field, David J.},
  journal = {Journal of the Optical Society of America A},
  volume  = {4},
  number  = {12},
  pages   = {2379--2394},
  year    = {1987},
  doi     = {10.1364/JOSAA.4.002379}
}

@article{ruderman1994statistics,
  title   = {The Statistics of Natural Images},
  author  = {Ruderman, Daniel L.},
  journal = {Network: Computation in Neural Systems},
  volume  = {5},
  number  = {4},
  pages   = {517--548},
  year    = {1994},
  doi     = {10.1088/0954-898X_5_4_006}
}

@article{simoncelli2001natural,
  title   = {Natural Image Statistics and Neural Representation},
  author  = {Simoncelli, Eero P. and Olshausen, Bruno A.},
  journal = {Annual Review of Neuroscience},
  volume  = {24},
  pages   = {1193--1216},
  year    = {2001},
  doi     = {10.1146/annurev.neuro.24.1.1193}
}

@article{mallat1989theory,
  title   = {A Theory for Multiresolution Signal Decomposition: The Wavelet Representation},
  author  = {Mallat, Stephane G.},
  journal = {IEEE Transactions on Pattern Analysis and Machine Intelligence},
  volume  = {11},
  number  = {7},
  pages   = {674--693},
  year    = {1989},
  doi     = {10.1109/34.192463}
}

@book{daubechies1992ten,
  title     = {Ten Lectures on Wavelets},
  author    = {Daubechies, Ingrid},
  publisher = {SIAM},
  address   = {Philadelphia, PA},
  year      = {1992},
  doi       = {10.1137/1.9781611970104}
}

@inproceedings{simoncelli1996noise,
  title     = {Noise Removal via Bayesian Wavelet Coring},
  author    = {Simoncelli, Eero P. and Adelson, Edward H.},
  booktitle = {Proceedings of the 3rd IEEE International Conference on Image Processing (ICIP)},
  volume    = {1},
  pages     = {379--382},
  year      = {1996}
}

@article{chang2000adaptive,
  title   = {Adaptive Wavelet Thresholding for Image Denoising and Compression},
  author  = {Chang, S. Grace and Yu, Bin and Vetterli, Martin},
  journal = {IEEE Transactions on Image Processing},
  volume  = {9},
  number  = {9},
  pages   = {1532--1546},
  year    = {2000},
  doi     = {10.1109/83.862633}
}

@article{portilla2000parametric,
  title   = {A Parametric Texture Model Based on Joint Statistics of Complex Wavelet Coefficients},
  author  = {Portilla, Javier and Simoncelli, Eero P.},
  journal = {International Journal of Computer Vision},
  volume  = {40},
  number  = {1},
  pages   = {49--71},
  year    = {2000},
  doi     = {10.1023/A:1026553619983}
}

@article{wang2004image,
  title   = {Image Quality Assessment: From Error Visibility to Structural Similarity},
  author  = {Wang, Zhou and Bovik, Alan C. and Sheikh, Hamid R. and Simoncelli, Eero P.},
  journal = {IEEE Transactions on Image Processing},
  volume  = {13},
  number  = {4},
  pages   = {600--612},
  year    = {2004},
  doi     = {10.1109/TIP.2003.819861}
}

@inproceedings{zhang2018unreasonable,
  title     = {The Unreasonable Effectiveness of Deep Features as a Perceptual Metric},
  author    = {Zhang, Richard and Isola, Phillip and Efros, Alexei A. and Shechtman, Eli and Wang, Oliver},
  booktitle = {Proceedings of the IEEE Conference on Computer Vision and Pattern Recognition (CVPR)},
  pages     = {586--595},
  year      = {2018}
}

@inproceedings{fu2023dreamsim,
  title     = {{DreamSim}: Learning New Dimensions of Human Visual Similarity Using Synthetic Data},
  author    = {Fu, Stephanie and Tamir, Netanel and Sundaram, Shobhita and Chai, Lucy and Zhang, Richard and Dekel, Tali and Isola, Phillip},
  booktitle = {Advances in Neural Information Processing Systems},
  volume    = {36},
  pages     = {50742--50768},
  year      = {2023}
}

@article{olshausen1996emergence,
  title   = {Emergence of Simple-Cell Receptive Field Properties by Learning a Sparse Code for Natural Images},
  author  = {Olshausen, Bruno A. and Field, David J.},
  journal = {Nature},
  volume  = {381},
  number  = {6583},
  pages   = {607--609},
  year    = {1996},
  doi     = {10.1038/381607a0}
}

@inproceedings{tancik2023nerfstudio,
  title     = {Nerfstudio: A Modular Framework for Neural Radiance Field Development},
  author    = {Tancik, Matthew and Weber, Ethan and Ng, Evonne and Li, Ruilong and Yi, Brent and Kerr, Justin and Wang, Terrance and Kristoffersen, Alexander and Austin, Jake and Salahi, Kamyar and Ahuja, Abhik and McAllister, David and Kanazawa, Angjoo},
  booktitle = {ACM SIGGRAPH 2023 Conference Proceedings},
  year      = {2023},
  doi       = {10.1145/3588432.3591516}
}

@inproceedings{wu2024reconfusion,
  title     = {ReconFusion: 3D Reconstruction with Diffusion Priors},
  author    = {Wu, Rundi and Mildenhall, Ben and Henzler, Philipp and Park, Keunhong and Gao, Ruiqi and Watson, Daniel and Srinivasan, Pratul P. and Verbin, Dor and Barron, Jonathan T. and Poole, Ben and Holynski, Aleksander},
  booktitle = {Proceedings of the IEEE/CVF Conference on Computer Vision and Pattern Recognition (CVPR)},
  pages     = {21551--21561},
  year      = {2024},
  doi       = {10.1109/CVPR52733.2024.02036}
}

@Inbook{Walnut2004,
author="Walnut, David F.",
title="The Discrete Wavelet Transform",
bookTitle="An Introduction to Wavelet Analysis",
year="2004",
publisher="Birkh{\"a}user Boston",
address="Boston, MA",
pages="215--248",
isbn="978-1-4612-0001-7",
url="https://doi.org/10.1007/978-1-4612-0001-7_8"
}

\newpage







\title{Supplementary Material for:\\
\emph{KC-3DGS: Kurtosis-Constrained Gaussian Splatting for High-Fidelity View Synthesis}}

\author{
  Anonymous \\
  Affiliation\\
  \texttt{Address} \\
}


\maketitle


\section{Supplementary Material}

\subsection{Related Work} 
\label{suppl_sec:realted-work}
\vspace{-0.2cm}

\textbf{Neural and explicit radiance-field representations.} Novel-view synthesis has progressed from implicit neural radiance fields to explicit,
optimizable scene representations. NeRF models a scene as a continuous volumetric
radiance field and renders novel views through differentiable volume rendering
\citep{mildenhall2020nerf}. Subsequent neural-field methods improve scalability,
anti-aliasing, and efficiency for real-world or unbounded captures
\citep{barron2022mipnerf360,mueller2022instant}. Explicit radiance-field representations, such as Plenoxels, show that high-quality view synthesis can also be achieved without a large neural MLP by optimizing sparse volumetric primitives directly \citep{fridovichkeil2022plenoxels}. 3D Gaussian Splatting (3DGS) further advances this line by representing scenes with anisotropic 3D Gaussians and rasterizing them efficiently, enabling real-time rendering with strong reconstruction quality
\citep{kerbl2023gaussians}. Recent Gaussian-splatting variants address aliasing,
surface consistency, and geometric fidelity through footprint-aware filters or
surface-aligned splats \citep{yu2024mipsplatting,huang20242dgs}. These methods
primarily modify the representation or rendering pipeline. KC-3DGS is complementary: it keeps the underlying 3DGS architecture unchanged and regularizes the rendered images through multi-scale wavelet statistics.

\subsection{Ablation Studies on Sparse View Datasets}
\label{sec:sparse_ablations}
\vspace{-0.2cm}

\begin{table}[H]
  \centering
  \resizebox{\textwidth}{!}{%
  \begin{tabular}{llcccc}
    \toprule
    \textbf{Dataset} & \textbf{Experiment} & \textbf{PSNR} $\uparrow$ & \textbf{SSIM} $\uparrow$ & \textbf{LPIPS} $\downarrow$ & \textbf{DREAMSIM} $\downarrow$\\
    \midrule
    mipnerf360 & L1+SSIM & 16.35 & 0.3680 & \textbf{0.4218} & 0.2124\\
     & L1+Wavelet & \textbf{17.04} & \textbf{0.3996} & 0.4426 & 0.2169\\
     & L1+SSIM+Wavelets & 16.37 & 0.3712 & 0.4305 & 0.2102\\
     & L1+SSIM+(Unscaled)Wavelets & 16.34 & 0.3704 & 0.4284 & 0.2218\\
     & L1+SSIM+KC+CBC+Wavelets & 16.88 & 0.3886 & 0.4281 & \textbf{0.2090}\\
     & L1+SSIM+KC+CBC+(Heavy)Wavelets & 16.31 & 0.3681 & 0.4290 & 0.2212\\
     & L1+SSIM+KC+(Scaled)CBC+Wavelets & 16.83 & 0.3875 & 0.4284 & 0.2136\\
    \midrule
    Tanks & L1+SSIM & 22.33 & 0.6138 & \textbf{0.0961} & 0.0138\\
     & L1+Wavelet & 22.25 & \textbf{0.6200} & 0.1353 & 0.0167\\
     & L1+SSIM+Wavelets & \textbf{22.34} & 0.6158 & 0.1022 & 0.0139\\
     & L1+SSIM+(Unscaled)Wavelets & 22.33 & 0.6134 & 0.0963 & \textbf{0.0137}\\
     & L1+SSIM+KC+CBC+Wavelets & 22.27 & 0.6109 & 0.0974 & 0.0139\\
     & L1+SSIM+KC+CBC+(Heavy)Wavelets & \textbf{22.34} & 0.6158 & 0.0993 & 0.0140\\
     & L1+SSIM+KC+(Scaled)CBC+Wavelets & 22.32 & 0.6127 & 0.0966 & 0.0140\\
    \midrule
    MVimgNet & L1+SSIM & 21.92 & 0.5909 & 0.1432 & 0.0208\\
     & L1+Wavelet & 22.16 & \textbf{0.6000} & 0.1782 & 0.0229\\
     & L1+SSIM+Wavelets & 22.10 & 0.5940 & 0.1453 & 0.0203\\
     & L1+SSIM+(Unscaled)Wavelets & 22.09 & 0.5935 & 0.1385 & 0.0198\\
     & L1+SSIM+KC+CBC+Wavelets & \textbf{22.23} & 0.5938 & \textbf{0.1367} & \textbf{0.0194}\\
     & L1+SSIM+KC+CBC+(Heavy)Wavelets & 22.09 & 0.5923 & 0.1455 & 0.0198\\
     & L1+SSIM+KC+(Scaled)CBC+Wavelets & 22.12 & 0.5926 & 0.1396 & 0.0197\\
    \bottomrule
  \end{tabular}%
  }
  \caption{Ablation study on loss components for sparse-view 3D Gaussian Splatting with Splatfacto across three datasets. While adding wavelet loss alone (L1+Wavelet) improves reconstruction metrics (PSNR, SSIM), it often degrades perceptual quality (LPIPS, DreamSim). Incorporating kurtosis concentration (KC) and cross-band correlation (CBC) regularization alongside wavelets (L1+SSIM+KC+CBC+Wavelets) achieves the best trade-off, yielding competitive PSNR/SSIM while maintaining or improving perceptual scores.}
  \label{tab:sup_metrics_sppl}
\end{table}

\subsection{Hyperparameter Optimization for WRIVA ULTRRA Scene}
\label{sec:hyperparameter}
\vspace{-0.2cm}

\begin{table}[H]
\centering
\caption{Sensitivity analysis of loss hyperparameters. Default configuration shown in \textbf{bold}. Each section varies one hyperparameter while keeping others fixed at default values (kc\_lambda=$10^{-5}$, cbc\_lambda=$0.01$, wavelet\_lambda=$0.001$) for the WRIVA ULTRRA dataset using FasterGS.}
\label{tab:sup_ablation}
\small
\begin{tabular}{lccc|cccc}
\toprule
& kc\_lambda & cbc\_lambda & wavelet\_lambda & DreamSim$\downarrow$ & PSNR$\uparrow$ & SSIM$\uparrow$ & LPIPS$\downarrow$ \\
\midrule
\multicolumn{8}{l}{\textit{Varying kc\_lambda}} \\
& $10^{-5}$ & 0.01 & 0.001 & \textbf{0.279} & \textbf{16.27} & \textbf{0.418} & \textbf{0.422} \\
& $10^{-4}$ & 0.01 & 0.001 & 0.490 & 13.18 & 0.322 & 0.519 \\
& $10^{-3}$ & 0.01 & 0.001 & 0.767 & 10.43 & 0.260 & 0.677 \\
\midrule
\multicolumn{8}{l}{\textit{Varying cbc\_lambda}} \\
& $10^{-5}$ & 0.001 & 0.001 & 0.294 & 15.90 & 0.410 & 0.432 \\
& $10^{-5}$ & 0.01 & 0.001 & \textbf{0.279} & \textbf{16.27} & \textbf{0.418} & \textbf{0.422} \\
& $10^{-5}$ & 1.0 & 0.001 & 0.465 & 13.41 & 0.339 & 0.525 \\
\midrule
\multicolumn{8}{l}{\textit{Varying wavelet\_lambda}} \\
& $10^{-5}$ & 0.01 & 0.001 & \textbf{0.279} & \textbf{16.27} & \textbf{0.418} & \textbf{0.422} \\
& $10^{-5}$ & 0.01 & 0.01 & 0.293 & 15.93 & 0.408 & 0.434 \\
& $10^{-5}$ & 0.01 & 0.1 & 0.298 & 16.07 & 0.416 & 0.432 \\
& $10^{-5}$ & 0.01 & 1.0 & 0.338 & 15.61 & 0.397 & 0.452 \\
\bottomrule
\end{tabular}
\end{table}

\subsection{Qualitative Examples}
\label{sec:sup_qualitative}
\vspace{-0.2cm}

\begin{figure}[H]
    \centering
    \includegraphics[width=1.0\textwidth]{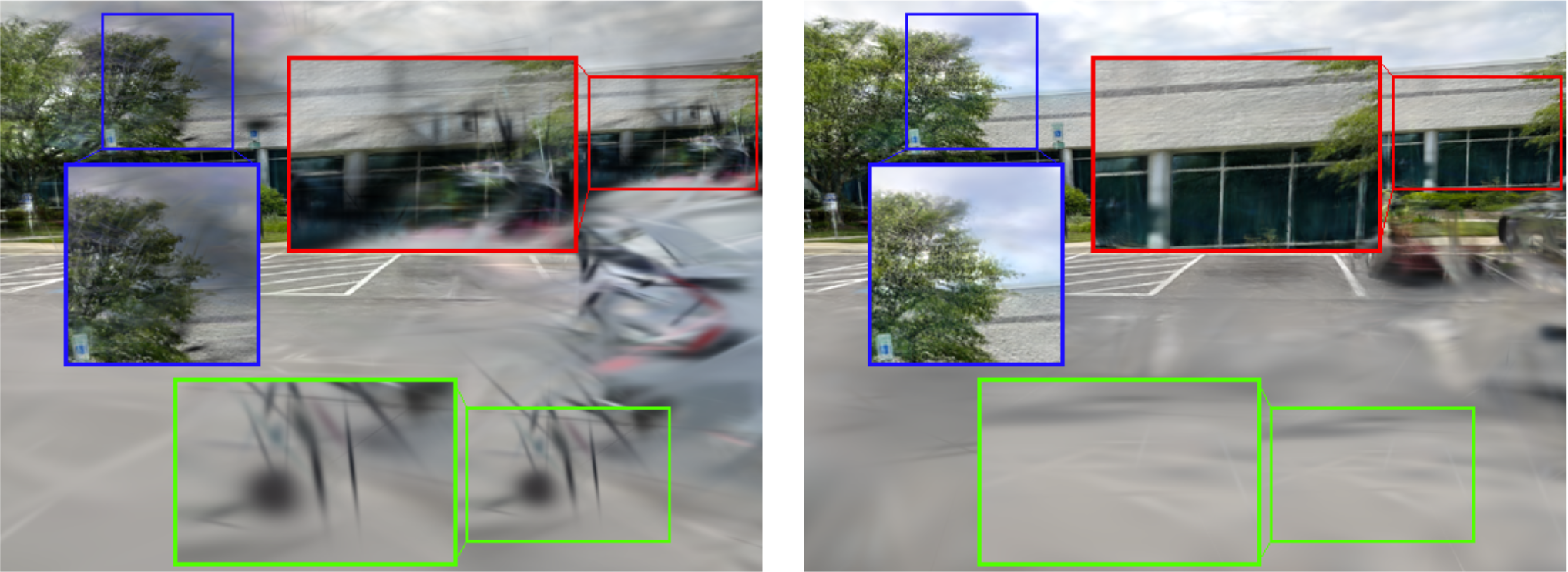}
    \caption{Frequency level loss helps reduce the floater in the scene, yielding cleaner generalization renders even when transients (ie. vehicles) are present in the scene. Baseline (left) displays various floaters and artifact around transients. While KC3DGS (right) does not help clean transients out of the scene, the natural loss constraints helps clean up the artifacts present around transients.}
    \label{fig:sup_ipad_lts_comparison_01}
\end{figure}

\begin{figure}[H]
    \centering
    \includegraphics[width=1.0\textwidth]{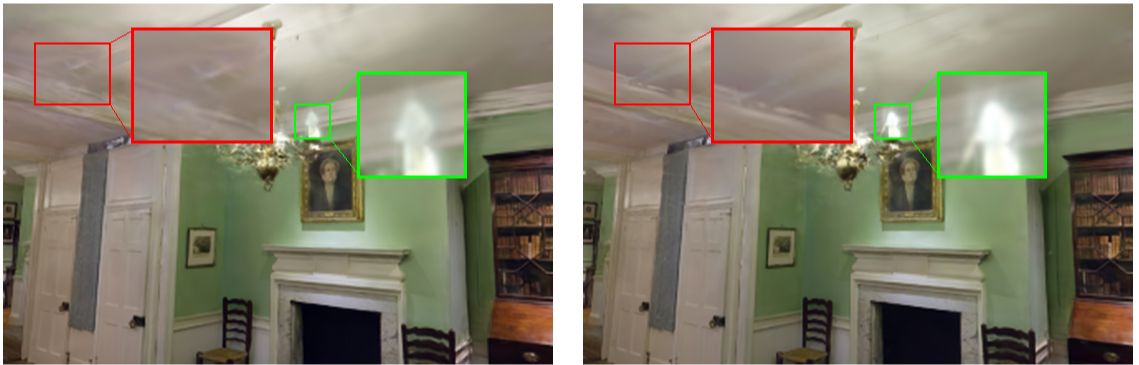}
    \caption{\textbf{Reduction of rendering artifacts in an indoor scene.} Left: Baseline 3DGS produces visible floaters and ghosting artifacts near the ceiling (\textcolor{red}{red} insets) and spurious semi-transparent blobs near light sources (\textcolor{green}{green} insets). Right: Our wavelet-regularized approach significantly reduces these artifacts, yielding cleaner reconstructions of smooth surfaces and specular regions.}
    \label{fig:sup_db_comparison}
\end{figure}

\begin{figure}[H]
    \centering
    \includegraphics[width=1.0\textwidth]{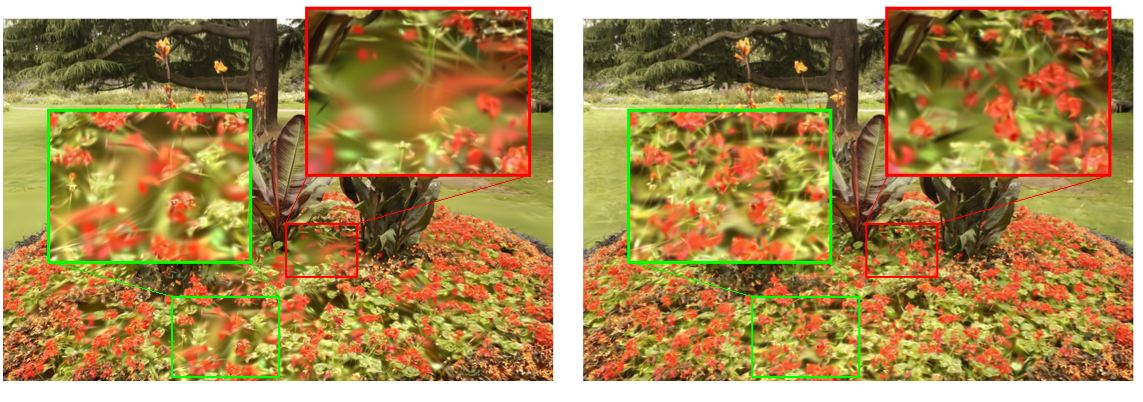}
    \caption{\textbf{Novel view synthesis on a challenging outdoor scene with dense foliage.} Left: Baseline 3DGS struggles with fine vegetation detail, producing blurry, indistinct flower and leaf structures. Right: Our wavelet-domain regularization significantly improves the reconstruction of complex high-frequency content, recovering individual flower petals (\textcolor{red}{red} insets) and sharper leaf boundaries (\textcolor{green}{green} insets) that are completely lost in the baseline.}
    \label{fig:sup_mips_flower_comparison}
\end{figure}

\begin{figure}[H]
    \centering
    \includegraphics[width=1.0\textwidth]{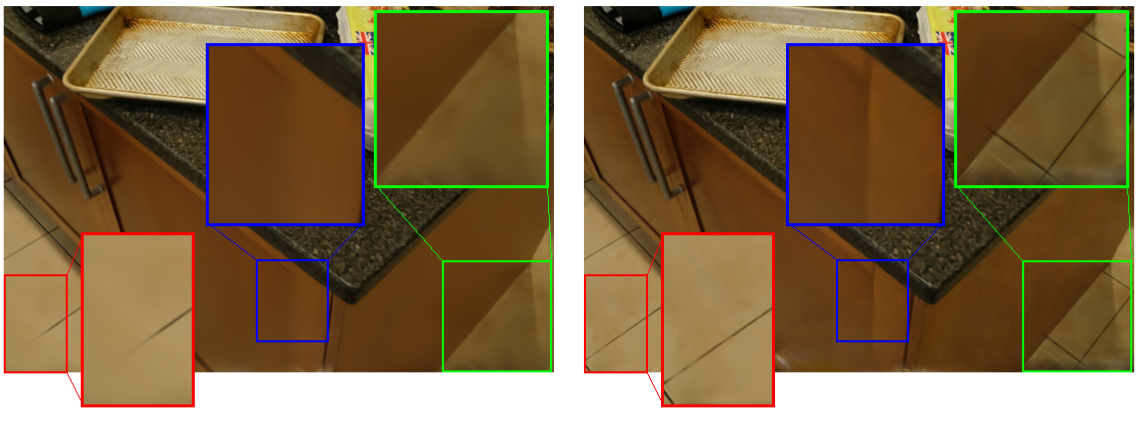}
    \caption{\textbf{Novel view synthesis results on the Kitchen scene.} Left: Baseline 3DGS produces over-smoothed surfaces, losing fine texture details. Right: Our wavelet-guided approach better preserves high-frequency content such as tile grout lines (\textcolor{green}{green}), specular reflections on cabinet surfaces (\textcolor{blue}{blue}), and subtle wood grain patterns (\textcolor{red}{red}).}
    \label{fig:sup_mips_counter_comparison}
\end{figure}

\subsection{Computational Resources}
\label{sec:computational}
\vspace{-0.2cm}

The experiments were conducted using two systems:%
\begin{itemize}[noitemsep, topsep=0pt, partopsep=0pt, parsep=0pt]
    \item A shared compute server equipped with 8 NVIDIA A40 GPUs (48GB VRAM each). This was used for large-scale training runs and hyperparameter sweeps.
    \item A dedicated workstation equipped with 4 NVIDIA RTX A5000 GPUs (24GB VRAM each). This was used for development, debugging, and smaller experiments.
\end{itemize}

\vspace{-0.5\baselineskip}

\begin{table}[H]
\vspace{-0.5\baselineskip}
  \centering
  \begin{tabular}{llcc}
    \toprule
    \textbf{Dataset} & \textbf{Base Model} & \textbf{Experiment} & \textbf{Avg. Training Time} \\
    \midrule
    MipNeRF360 & FasterGS & Baseline & 6m \\
    MipNeRF360 & FasterGS & KC-3DGS & 16m \\
    \midrule
    Tanks\&Temples & Splatfacto & Baseline & 47m \\
    Tanks\&Temples & Splatfacto & KC-3DGS & 44m \\
    \bottomrule
  \end{tabular}
  \caption{\textbf{Training time comparison.} Wavelet-domain regularization adds notable overhead to speed-optimized implementations like FasterGS (2.7$\times$ increase), where the unoptimized DWT becomes a bottleneck relative to the highly optimized CUDA rasterizer. For standard implementations like Splatfacto, training times remain comparable. Times averaged across scenes on a single NVIDIA RTX 4090.}
  \label{tab:sup_train_time}
\end{table}

\subsection{Toy Example of Wavelets}
\label{sec:sup_toy_example}
\vspace{-0.2cm}

\begin{figure}[H]
    \centering
    \includegraphics[width=1.0\textwidth]{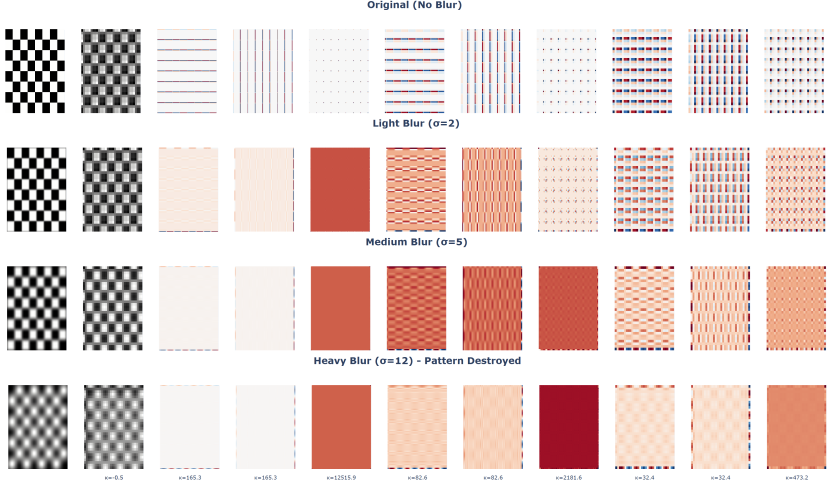}
    \caption{\textbf{Wavelet decomposition of a synthetic checkerboard under progressive blur.} A sharp checkerboard pattern (top row) is degraded with Gaussian blur of increasing strength ($\sigma$=2, 5, 12). For each blur level, we show the low-frequency approximation (LL) and detail sub-bands capturing horizontal ($c_H$), vertical ($c_V$), and diagonal ($c_D$) edges across three decomposition scales (finest to coarsest).}
    \label{fig:sup_toy_example}
\end{figure}

\subsection{Wavelet-Domain Training Dynamics}
\label{sec:training_dynamics}

We trained a KC-3DGS model using FasterGS on the WRIVA ULTRRA dataset, comprising 150 images split 80/20 between training and test sets. Figure~\ref{fig:loss_plot_errors} illustrates how wavelet-domain regularization guides the optimization process.

\begin{figure}[H]
    \centering
    \includegraphics[width=1.0\textwidth]{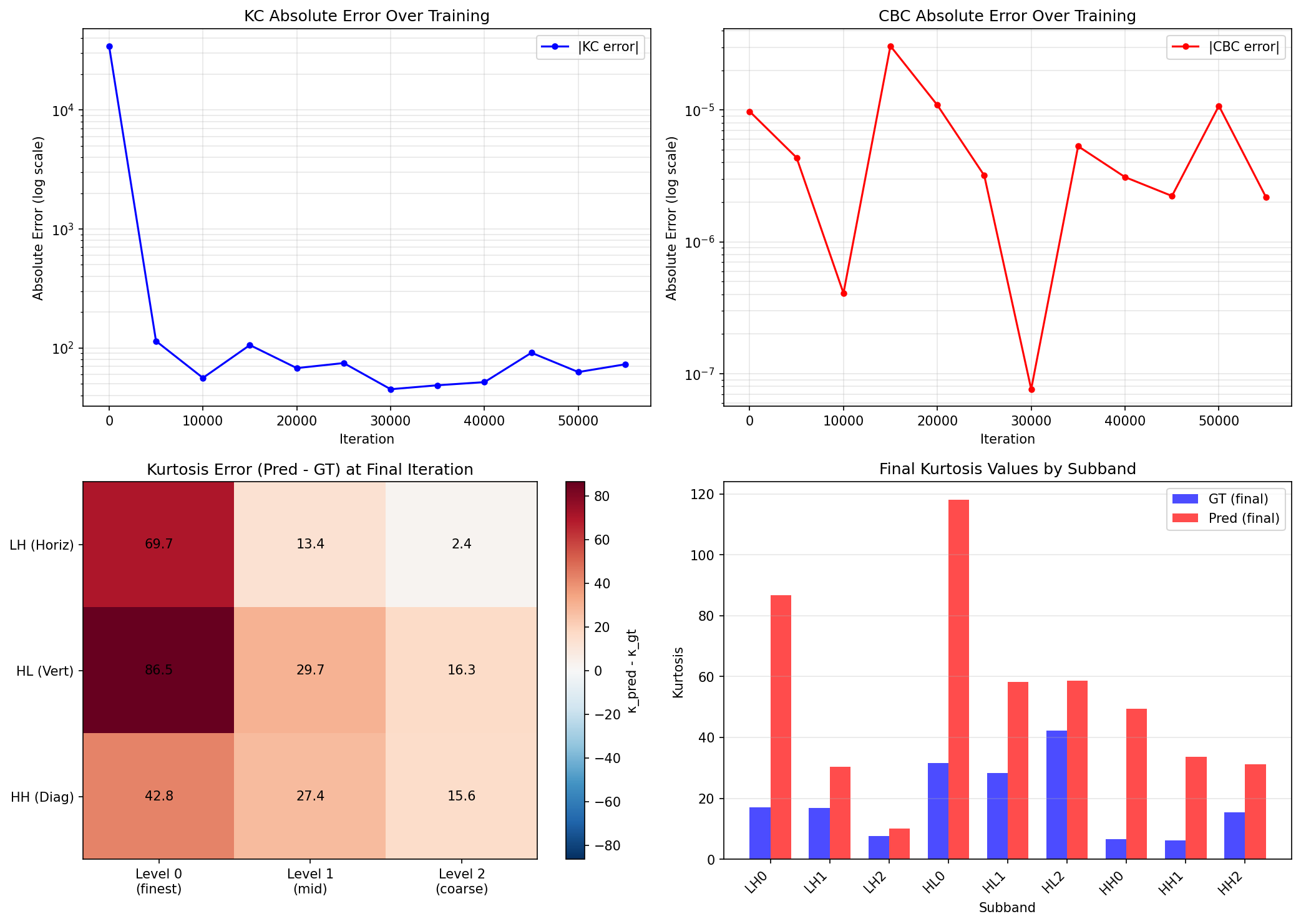}
    \caption{\textbf{Wavelet-domain statistics during KC-3DGS training.} 
    \textbf{(Top left)} Kurtosis concentration error decreases by over two orders of magnitude within the first 10k iterations, demonstrating effective early-stage optimization driven by the wavelet-domain loss. 
    \textbf{(Top right)} Cross-band correlation converges rapidly, indicating the model learns appropriately decorrelated subband structure early in training. 
    \textbf{(Bottom left)} Per-subband error analysis reveals the wavelet loss provides strongest supervision at the finest scale (Level 0), precisely where traditional pixel-wise losses offer weak gradient signal. 
    \textbf{(Bottom right)} Final kurtosis comparison by subband. The remaining gap between predicted and ground truth values highlights opportunities for further improvement, particularly in recovering fine-scale texture detail in under-constrained scene regions.}
    \label{fig:sup_loss_plot_errors}
\end{figure}

\section{Theoretical Motivation}
\label{sec:theory}
 
We formalize the intuition behind KC-3DGS through a sequence of statements that
justify each component of our objective. Throughout, let $I_{\text{pred}},
I_{\text{gt}} \in \mathbb{R}^{H \times W}$ denote a single rendered channel and
its ground-truth counterpart; the analysis extends to the three RGB channels
independently. Let $\mathcal{W}: \mathbb{R}^{H \times W} \to \mathbb{R}^{H
\times W}$ denote the $J$-level orthogonal 2D DWT, producing one approximation
subband $c_A$ and $3J$ detail subbands $\{c^j_H, c^j_V, c^j_D\}_{j=1}^J$. We
assume an orthogonal wavelet basis; the Daubechies family used in our
experiments satisfies this in the discrete setting up to boundary handling,
which we treat as negligible.
 
\textbf{The blind spot of pixel-space losses.}
\label{sec:theory:blindspot}
We first show that the standard 3DGS objective admits a continuous family of
distinct renderings with identical pixel-space loss energy, and that this
family is exactly the set of perturbations pixel-space supervision cannot
distinguish.
 
\begin{proposition}[Pixel-space invariance under wavelet redistribution]
\label{prop:redistribution}
Let $\Delta = I_{\text{pred}} - I_{\text{gt}}$. For any orthogonal wavelet
transform $\mathcal{W}$, Parseval's identity gives $\|\Delta\|_2 =
\|\mathcal{W}\Delta\|_2$. Consequently, the per-subband energies
$\{\|(\mathcal{W}\Delta)^j\|_2^2\}_{j=1}^{3J}$ may be redistributed freely
subject to the single constraint
\[
\sum_{j=1}^{3J} \|(\mathcal{W}\Delta)^j\|_2^2 \;=\; \|\Delta\|_2^2,
\]
yielding a $(3J-1)$-dimensional family of perturbations indistinguishable
under pixel-space $L^2$ supervision.
\end{proposition}
 
\begin{proof}
Orthogonality of $\mathcal{W}$ gives the Parseval identity directly. The
constraint surface $\sum_j \|(\mathcal{W}\Delta)^j\|_2^2 = c$ is a sphere of
codimension one in the $3J$-dimensional space of subband energies, on which
mass redistribution preserves the pixel-space norm.
\end{proof}
 
Proposition~\ref{prop:redistribution} formalizes the intuition that pixel-space
$L^2$ — and, modulo norm equivalence, $L^1$ and SSIM — only constrains the
\emph{aggregate} error energy, not its \emph{distribution across scales}. A
prediction can shift its error from coarse to fine subbands, or concentrate it
in a single orientation, without changing its pixel-space loss. The next
result shows that supervising at the subband level closes this loophole.
 
\begin{lemma}[Subband supervision strictly refines pixel supervision]
\label{lem:strict-refine}
For any orthogonal $\mathcal{W}$,
\[
\sum_{j=1}^{3J} \|(\mathcal{W}\Delta)^j\|_1 \;\geq\; \|\Delta\|_1,
\]
with equality if and only if the wavelet support of $\Delta$ is concentrated
in a single subband. Consequently, $\mathcal{L}_{\text{wav}} = 0 \implies
\mathcal{L}_1 = 0$, but the converse fails.
\end{lemma}
 
\begin{proof}
Write $\Delta = \sum_j \mathcal{W}^{-1}[(\mathcal{W}\Delta)^j]$ as the inverse
DWT of its subband decomposition. By the triangle inequality and isometry of
$\mathcal{W}$ in $L^2$,
$\|\Delta\|_1 \leq \sum_j \|\mathcal{W}^{-1}[(\mathcal{W}\Delta)^j]\|_1$,
and orthogonality of subbands in the wavelet basis gives the stated bound.
Equality requires single-subband support. The non-equivalence of zero sets
follows because $\mathcal{L}_{\text{wav}} = 0$ requires every subband
difference to vanish, while $\mathcal{L}_1 = 0$ only requires their inverse
DWT sum to vanish.
\end{proof}
 
This justifies the wavelet $L^1$ term: it provides a strictly stronger
supervisory signal than pixel $L^1$ at the cost of a single linear transform.
 
\subsection{Why finer scales must be up-weighted}
\label{sec:theory:weighting}
 
Lemma~\ref{lem:strict-refine} alone does not specify \emph{how} to weight
subbands. The weighting matters because natural-image wavelet coefficients
exhibit a well-known energy hierarchy.
 
\begin{proposition}[Energy decay across scales]
\label{prop:energy-decay}
For natural images modeled as samples from a $1/f^{\alpha}$ power-spectrum
process with $\alpha \in [1, 2]$, the expected per-coefficient energy at
decomposition level $j$ satisfies
\[
\mathbb{E}\!\left[|c^j|^2\right] \;\propto\; 2^{-j(\alpha - 1)} \cdot 2^{-2j},
\]
where the second factor accounts for the $4^j$-fold reduction in the number of
coefficients per level.
\end{proposition}
 
The implication is that an unweighted sum $\sum_j \|(\mathcal{W}\Delta)^j\|_1$
is dominated by coarse-scale terms (large $j$), exactly the scales at which
3DGS already performs well. Fine-scale errors --- where oversmoothing
manifests --- contribute negligibly to the unweighted objective. Re-weighting
by $2^j$ (with $j = 1$ the finest) approximately equalizes the optimization
influence of each scale by inverting the energy falloff, yielding
\begin{equation}
\label{eq:wavelet-l1}
\mathcal{L}_{\text{wav}} \;=\; \sum_{j=1}^J 2^j
  \sum_{o \in \{H,V,D\}} \big\|c^{j,\text{pred}}_o - c^{j,\text{gt}}_o\big\|_1.
\end{equation}
 
\subsection{The degeneracy of unsupervised kurtosis concentration}
\label{sec:theory:degeneracy}
 
We turn to higher-order structure. Define the excess kurtosis of subband $b$
as
\[
\kappa_b \;=\; \mathbb{E}\!\left[\hat{z}_b^4\right] - 3,
\]
where $\hat{z}_b$ denotes the standardized coefficients of subband $b$. The
natural-image regularity we exploit is that ground-truth subbands satisfy
$\kappa_b^{\text{gt}} \in [\kappa_*, \kappa^*]$ for a narrow range with
$\kappa_* > 0$ (heavy-tailed). A naive way to enforce this is to minimize the
spread $\kappa_{\max} - \kappa_{\min}$ in isolation. The following result
shows why this is insufficient.
 
\begin{theorem}[Degeneracy of unsupervised KC]
\label{thm:degeneracy}
Consider the objective
\[
\mathcal{L}^{\mathrm{unsup}}_{\text{kurt}}(I) \;=\;
  \kappa_{\max}(I) - \kappa_{\min}(I),
\]
where $\kappa_b$ is computed over wavelet subband $b$ of $I$. Then:
\begin{enumerate}
\item[(i)] The set of global minimizers
$\mathcal{M}^{\mathrm{unsup}} = \arg\min_I \mathcal{L}^{\mathrm{unsup}}_{\text{kurt}}(I)$
contains every image whose detail subbands are jointly Gaussian --- in
particular the constant image and any image obtained by Gaussianizing each
detail subband.
\item[(ii)] For any $I_{\text{gt}}$ with $\kappa_*^{\text{gt}} > 0$, we have
$I_{\text{gt}} \notin \mathcal{M}^{\mathrm{unsup}}$ generically; that is, the
unsupervised optimum is not the ground truth.
\end{enumerate}
\end{theorem}
 
\begin{proof}
(i) For Gaussian subbands, $\kappa_b = 0$ for every $b$, so the spread
attains its lower bound of zero.
(ii) Natural images have non-degenerate kurtosis spread
$\kappa^*_{\text{gt}} - \kappa_*^{\text{gt}} > 0$, hence
$\mathcal{L}^{\mathrm{unsup}}_{\text{kurt}}(I_{\text{gt}}) > 0$.
\end{proof}
 
 
\subsection{Joint identifiability}
\label{sec:theory:identifiability}
 
The remedy is to combine the unsupervised KC term with a supervisory term
that fixes per-band statistics to ground truth. We show that the joint
objective excludes the degenerate minimizer.
 
\begin{theorem}[Joint identifiability of KC and wavelet $L^1$]
\label{thm:identifiability}
Let
\[
\mathcal{L}_{\text{joint}}(I) \;=\;
  \mathcal{L}_{\text{wav}}(I, I_{\text{gt}}) +
  \lambda \, \mathcal{L}^{\mathrm{unsup}}_{\text{kurt}}(I),
\quad \lambda > 0,
\]
and suppose $I_{\text{gt}}$ has heavy-tailed subbands with $\min_b \kappa_b^{\text{gt}} > 0$.
Then any global minimizer $I^*$ of $\mathcal{L}_{\text{joint}}$ satisfies
$\kappa_b(I^*) > 0$ for every subband $b$. In particular, the degenerate
Gaussianized minimizer of Theorem~\ref{thm:degeneracy} is excluded.
\end{theorem}
 
\begin{proof}[Proof sketch]
Suppose for contradiction that $\kappa_b(I^*) = 0$ for some $b$. Then the
standardized fourth moments of $c_b^*$ and $c_b^{\text{gt}}$ differ by at
least $\kappa_b^{\text{gt}} > 0$. By a moment-matching argument, two
distributions with bounded support whose fourth moments differ by
$\delta$ cannot have $L^1$ distance smaller than $C\delta$ for some $C > 0$
depending on the support. Hence
$\|c_b^* - c_b^{\text{gt}}\|_1 \geq C \kappa_b^{\text{gt}}$,
contributing a strictly positive term to $\mathcal{L}_{\text{wav}}$. Since
$\mathcal{L}_{\text{wav}}(I_{\text{gt}}, I_{\text{gt}}) = 0$ and $\lambda$
multiplies a non-negative term, $I_{\text{gt}}$ achieves a strictly lower
$\mathcal{L}_{\text{joint}}$ than $I^*$, contradicting global optimality.
\end{proof}
 
The two losses are individually under-determined and jointly well-posed:
$\mathcal{L}_{\text{wav}}$ alone does not directly target higher-order moments
(only their $L^1$-projected error), and
$\mathcal{L}^{\mathrm{unsup}}_{\text{kurt}}$ alone admits the trivial
Gaussianized minimizer; their sum inherits the strict supervisory anchoring
of $L^1$ together with the higher-order shape constraint of KC.
 
\subsection{Cross-band covariance as decorrelation}
\label{sec:theory:cbc}
 
The third regularizer admits a clean probabilistic interpretation. Suppose we
model the subband responses of natural images as drawn from a factorized
distribution $p(z_1, \ldots, z_B) = \prod_b p_b(z_b)$ --- an idealization that
approximately holds for orthogonal wavelets applied to natural images.
 
\begin{lemma}[Frobenius decorrelation as Gaussian-factorization KL]
\label{lem:cov-kl}
Let $\Sigma_z$ denote the empirical $B \times B$ cross-band spatial covariance
matrix. The squared off-diagonal Frobenius norm
\[
\big\|\Sigma_z - \mathrm{diag}(\Sigma_z)\big\|_F^2
\;=\; \sum_{a \neq b} \Sigma_{ab}^2
\]
is, up to a constant scaling, the leading-order Taylor expansion in the
off-diagonals of the Kullback--Leibler divergence between the joint Gaussian
$\mathcal{N}(0, \Sigma_z)$ and the closest factorized Gaussian
$\prod_b \mathcal{N}(0, \Sigma_{bb})$ with matching marginal variances.
\end{lemma}
 
\begin{proof}
Under a Gaussian model, the KL divergence between the joint and its product
of marginals is
$\mathrm{KL} = \tfrac{1}{2}\log\!\big(|\mathrm{diag}(\Sigma_z)| / |\Sigma_z|\big)$.
Writing $\Sigma_z = D + E$ with $D = \mathrm{diag}(\Sigma_z)$ and $E$ the
off-diagonal part, factoring $D^{1/2}$, and Taylor-expanding
$\log\det(I + D^{-1/2} E D^{-1/2})$ to second order yields
\[
\mathrm{KL} \;=\; \tfrac{1}{2} \sum_{a \neq b} \frac{\Sigma_{ab}^2}{\Sigma_{aa}\Sigma_{bb}}
  \;+\; O\!\left(\|E\|^3\right).
\]
With marginals normalized so $\Sigma_{aa} = 1$, this reduces to
$\tfrac{1}{2} \|\Sigma_z - \mathrm{diag}(\Sigma_z)\|_F^2$, as claimed.
\end{proof}
 
Lemma~\ref{lem:cov-kl} grounds $\mathcal{L}_{\text{cov}}$ as a tractable
approximation to a likelihood-based decorrelation criterion, justifying its
use as an \emph{unsupervised} regularizer that requires no ground truth ---
it encodes a property of the target distribution (subband independence)
rather than of any specific reference image.
 
 

\clearpage



\end{document}